\documentclass[11pt,letterpaper]{article}
\usepackage{amsmath,amsthm,nicefrac}
\usepackage{amsfonts, amstext}

\usepackage{comment}

\usepackage{typearea}
\typearea{14}
\usepackage[font={small,it}]{caption}

\usepackage{setspace}

\usepackage{enumitem}
\usepackage[breaklinks]{hyperref}
\hypersetup{colorlinks=true,%
            citebordercolor={.6 .6 .6},linkbordercolor={.6 .6 .6},%
citecolor=blue,urlcolor=black,linkcolor=blue,pagecolor=black}

\usepackage[nameinlink]{cleveref}
\Crefname{algocf}{Algorithm}{Algorithms}
\crefname{algocfline}{line}{lines}
\Crefname{invariant}{Invariant}{Invariants}
\Crefname{claim}{Claim}{Claims}
\Crefname{subclaim}{Subclaim}{Subclaims}

\usepackage{epsfig}
\usepackage{amsthm,amssymb}
\usepackage{amsthm,amssymb}

\usepackage{mathrsfs}
\usepackage{xspace}
\usepackage{soul}
\usepackage{latexsym}
\usepackage{bbm}

\usepackage{framed}

\usepackage[dvipsnames]{xcolor}
\definecolor{DarkGray}{rgb}{0.66, 0.66, 0.66}
\definecolor{DarkPowderBlue}{rgb}{0.0, 0.2, 0.6}
\definecolor{fluorescentyellow}{rgb}{0.8, 1.0, 0.0}

\usepackage[ruled,vlined,linesnumbered,algonl]{algorithm2e}
\SetEndCharOfAlgoLine{}
\SetKwComment{Comment}{\footnotesize$\triangleright$\ }{}

\SetCommentSty{mycommfont}

\usepackage{thmtools,thm-restate}

\allowdisplaybreaks

\newcounter{note}[section]

\sethlcolor{fluorescentyellow}

\newcommand{\initOneLiners}{%
    \setlength{\itemsep}{0pt}
    \setlength{\parsep }{0pt}
    \setlength{\topsep }{0pt}
}

\newtheorem{theorem}{Theorem}[section]
\newtheorem{lemma}[theorem]{Lemma}
\newtheorem{claim}[theorem]{Claim}
\newtheorem{fact}[theorem]{Fact}
\newtheorem{corollary}[theorem]{Corollary}

\theoremstyle{definition}
\newtheorem{defn}[theorem]{Definition}

\theoremstyle{remark}
\newtheorem{remark}[theorem]{Remark}

\makeatletter 
\renewcommand{\theinvariant}{(I\@arabic\c@invariant)}
\makeatother

\newcommand{\opt}{{\textsf{opt}}}

\newcommand{\poly}{\operatorname{poly}}

\renewcommand{\emptyset}{\varnothing}

\newcommand{\CH}{{\mathsf {CH}}}

\newcommand{\dist}{{\mathsf{dist}}}

\newcommand{\junk}[1]{}
\newcommand{\eat}[1]{}

\newif\ifhideproofs

\ifhideproofs
\usepackage{environ}
\NewEnviron{hide}{}

\fi

\newcommand{\cE}{{\cal E}}

\newcommand{\dst}{{\mathsf{dist}}}

\newcommand{\prob}{{\mathbf{Pr}}}
\newcommand{\lspan}{{\mathsf{Span}}}

\newcommand{\ch}{{\mathsf{CH}}}
\newcommand{\Haus}{{\mathsf{Haus}}}
\newcommand{\dia}{\Delta}
\newcommand{\eenv}{{\varepsilon{{\mathsf{\mbox{-}ENV}}}}}

\newcommand{\edenv}{{(\varepsilon, \delta){{\mathsf{\mbox{-}ENV}}}}}
\newcommand{\edenvp}{{(\varepsilon', \delta'){{\mathsf{\mbox{-}ENV}}}}}

\newcommand{\ech}{{\varepsilon{{\mathsf{\mbox{-}CH}}}}}
\newcommand{\echp}{{\varepsilon'{{\mathsf{\mbox{-}CH}}}}}

\newcommand{\bM}{{\bf M}}
\newcommand{\bP}{{\bf P}}
\newcommand{\bA}{{\bf A}}

\newcommand{\hatM}{{\widehat M}}
\newcommand{\hatA}{{\widehat A}}
\newcommand{\hatP}{{\widehat P}}
\newcommand{\hatK}{{\widehat K}}

\newcommand{\uell}{{u^{(\ell)}}}

\newcommand{\lkp}{{\mathsf {LkP}}\xspace}
\newcommand{\rsh}{{\mathsf {RSH}}\xspace}
\newcommand{\sht}{{\mathsf {SHT}}\xspace}
\newcommand{\olp}{{\mathsf {OLP}}\xspace}

\newcommand{\kolp}{k\mbox{-}\olp}

\newcommand{\ListLearn}{\mathsf {ListLearn}\xspace}

\newcommand{\Hausdorf}{\mathsf {Hausdorff}\xspace }

\newcommand{\Mdot}[1]{{M_{\cdot, #1}}}
\newcommand{\hMdot}[1]{{\hatM_{\cdot, #1}}}
\newcommand{\Pdot}[1]{{P_{\cdot, #1}}}
\newcommand{\Adot}[1]{{A_{\cdot, #1}}}
\newcommand{\hAdot}[1]{{\hatA_{\cdot, #1}}}
\newcommand{\hPdot}[1]{{\hatP_{\cdot, #1}}}
\newcommand{\sep}{{\mathsf{SepOr}}}
\newcommand{\opto}{{\mathsf{OptOr}}}
\begin{document}
\title{Random Separating Hyperplane Theorem and Learning Polytopes}
\author{Chiranjib Bhattacharyya \and Ravindran Kannan \and Amit Kumar}
\maketitle
\begin{abstract}
The Separating Hyperplane theorem  is a fundamental result in Convex Geometry with myriad applications. The theorem asserts that for a point $a$ not in a closed convex set $K$, there is a hyperplane with $K$ on one side and $a$ strictly on the other side. Our first result, Random Separating Hyperplane  Theorem ($\rsh$), is a strengthening of this for polytopes. $\rsh$ asserts that if the distance between $a$ and a polytope $K$ with $k$ vertices and unit diameter in $\Re^d$ is at least $\delta$, where $\delta$ is a fixed constant in $(0,1)$, then a randomly chosen hyperplane separates $a$ and $K$ with probability at least $1/\poly(k)$ and  margin at least $\Omega \left( \delta/\sqrt{d} \right)$.
There is  a rich body of work on reductions between (approximate) optimization and (approximate) separation oracles for general convex sets, where the focus has been on the number of oracle calls. 
An immediate consequence of our result is the first near optimal bound on the error increase in the reduction from a Separation oracle to an Optimization oracle over a polytope.

$\rsh$  has algorithmic applications in learning polytopes. We consider a fundamental problem, denoted the ``$\Hausdorf$ problem'', of learning a unit diameter polytope $K$ within Hausdorff distance $\delta$,  given an optimization oracle for $K$.
Using $\rsh$, we show that with polynomially many random queries to the optimization oracle, $K$ can be approximated within error  $O(\delta)$. 
To our knowledge this is the first provable algorithm for the $\Hausdorf$ Problem.  Building on this result, we show that if the vertices of $K$ are well-separated, then an optimization oracle can be used to generate a list of points, each within Hausdorff distance $O(\delta)$  of $K$, with the property that the list contains a point close to each vertex of $K$. %[We call this ``list-learning'' the vertices of $K$.] 
Further, we show how to prune this list to generate a (unique) approximation to each vertex of the polytope. We prove that in many latent variable settings, e.g., topic modeling, LDA,  optimization oracles do exist provided we project to a suitable SVD subspace. Thus, our work  yields the first efficient algorithm for finding approximations  to the vertices of the latent polytope under the well-separatedness assumption. This assumption states that each vertex of  $K$ is far from the convex hull of the remaining vertices of $K$, and is much weaker than other assumptions behind algorithms in the literature which find vertices of the latent polytope. 
\end{abstract}
\thispagestyle{empty}
\newpage
\setcounter{page}{1}
\section{Introduction}
The Separating Hyperplane theorem ($\sht$) is a fundamental result in Convex Geometry with myriad applications (see e.g.~\cite{bertsekas2009convex}). The theorem asserts that for a point $a$ not in a closed convex set $K$, there is a hyperplane with $K$ on one side and $a$ strictly on the other side.

This paper makes two main contributions.
Our theoretical contribution, which is an extension of the classical $\sht$,  is what we call the Random Separating Hyperplane Theorem ($\rsh$). Our main algorithmic contribution is to use $\rsh$ to prove that a natural algorithm, which we call the $\kolp$ algorithm, can learn (vertices of) latent polytopes arising in a number of problems in Latent Variable Models including Clustering, Mixture Learning , LDA (linear discriminant analysis), Topic Models.  The algorithmic result is shown by reducing the problem of learning latent polytopes in a variety of settings to that of constructing approximate optimization oracles for the corresponding polytopes. Bulk of our algorithmic contribution is in proving this reduction and the existence of such oracles.

% starts with a reduction of the problem of learning latent polytopes in a variety of problems to learning latent polytopes given by approximate optimization oracles. The reduction builds on earlier work in the literature. Bulk of our algorithmic contribution is in learning latent polytopes given by oracls.

% \aknote{I think this whole paragraph can go away} In the next section, we describe the theoretical setting. The rest of the paper focusses on the algorithmic application. We first show (see Section (\ref{algorithmic-rsh}) that if are given oracles with errors bounded by $O(1/\sqrt d)$, then, using $\rsh$, we can prove polynomial time learnability. But ....$1/\sqrt k$.... SVD... Section ???

%\subsection{Theoretical Setting}
%\aknote{Shortening this section and stating RSH upfront}
\vspace*{-0.1in}
\subsection{Random Separating Hyperplane Theorem($\rsh$)}
$\rsh$ draws its motivation mainly from the Separating Hyperplane Theorem ($\sht$) of Convex Geometry. It also has connections to the Johnson-Lindenstrauss Random Projection theorem~\cite{jl} and reductions among oracles in Convex Optimization. 
%
%\subsubsection{Sepaarting Hyperplane Theorem $\sht$}
% The Separating Hyperplane theorem ($\sht$) is a fundamental result in Convex Geometry with myriad applications (see e.g.~\cite{bertsekas2009convex}). The theorem asserts that for a point $a$ not in a closed convex set $K$, there is a hyperplane with $K$ on one side and $a$ strictly on the other side.
%
The $\sht$ formally states that given a closed convex set $K$ and a point $a \notin K$, there exists
 a (non-zero) vector $u$ such that 
\begin{equation*} %\label{121}
u\cdot a \, > \, \mbox{Max}_{y\in K} \, u\cdot y.
\end{equation*}
The following question arises: Does a randomly picked $u$ separate $a$ from $K$? Taking into account some necessary conditions for a positive answer, we can ask if the following inequality holds for a randomly chosen $u$: 
(here $\Delta$ is  the diameter of $K$, $a$ is 
at distance at least $\delta\Delta$
from $K$, where $\delta \in (0,1)$):
%
%Does the following inequality hold ?
\begin{equation}\label{SHT-Q}
\prob\left( u\cdot a\geq \mbox{Max}_{y\in K}u\cdot y+|u|\alpha\delta\Delta\right)\; \geq \; 1/\mbox {poly}_\delta,\footnote{poly$_\delta(z)$ denotes $z^{\poly(1/\delta)}.$
%a function which is a polynomial for each fixed value of $\delta.$
}
\end{equation}
with $\alpha$ being as high as possible. 

The question \eqref{SHT-Q} is also motivated from the Johnson-Lindenstrauss Random Projection theorem (JL theorem)~\cite{jl} which states that if $a,b$ are points in ${\bf R}^d$ and $U$ is a random subspace of dimension $s$, then with probability bounded away from 0, the distance between the projection of $a$ and $b$ on $U$ is at least 
 $\Omega(|a-b|\sqrt s/\sqrt d)$. 
 The following natural generalization of this is interesting already for $s=1$: 

Instead of $b$ being a point, if it is now a polytope $K$, does a  similar lower bound on the distance of $a$ to $K$ in the projection onto a random line hold?  

It is easy to see that in spirit, this is the same question as whether \eqref{SHT-Q} holds.
It is also easy (see below) to see that the projection shrinks the distance between $a$ and $K$ by a factor of $\Omega^*(\sqrt d)$. 
The $\rsh$ theorem proves that the shrinkage is $O(\sqrt d)$ thus making this parameter nearly (within log factors) tight. We now state  the $\rsh$ theorem:
\begin{theorem}
[Random Separating Hyperplane Theorem($\rsh$): Informal version]  
\label{rsh-0} Suppose $K$ is a $k$ vertex polytope with diameter $\Delta$ and $a$ is a point at distance at least $\delta\Delta,$ $\delta\in (0,1)$, from 
$K$. Let $V$ be an $m-$dimensional subspace containing $K\cup \{ a\}$. For a random Gaussian vector $u\in V$, the following event happens with probability  at least 
1/poly$_\delta (k)$:
\begin{align}
\label{eq:rsh}
u\cdot a\geq \mbox{Max}_{y\in K}u\cdot y + \frac{\delta \Delta |u|}{10\sqrt m}.
\end{align}
\end{theorem}
We provide a simple example where $K$ is a line segment~(see \Cref{sec:examples}) to show that the factor $\sqrt m$ cannot be improved. It is also interesting to note that the success probability of the event in~\eqref{eq:rsh} needs to depend on $k$. 
(see~\Cref{sec:examples}).
In particular, $\rsh$ does not hold for general convex sets (where $k$ is not necessarily finite.) \\

% \noindent
% {\bf{Connections with Johnson Lindenstrauss (JL)  Random Projection Theorem.}}
% The JL theorem states that if $a,b$ are points in ${\bf R}^d$ and $U$ is a random subspace of dimension $s$, then with probability bounded away from 0, the distance between the projection of $a$ and $b$ on $U$ is 
% %we have that in the projection onto $U$, $a,b$ are separated by at least 
% $\Omega(|a-b|\sqrt s/\sqrt d)$. A natural generalization of this is interesting already for $s=1$: Instead of $b$ being a point, if it is  a convex  polytope $K$, does a  similar lower bound on the distance between $a$ and $K$ in the projection onto a random line hold? It is easy to see that in spirit, this is the same question as whether \eqref{SHT-Q} holds. But the JL viewpoint suggests (and is easy to prove with simple examples, see~\Cref{sec:examples}) that $\alpha$ has to be $O(1/\sqrt d)$ and we  achieve this in $\rsh$. Also, as opposed to JL, which is a high probability ($1-o(1)$) result,   we seek only a 1/poly lower bound on the success probability. It is also easy to see that we cannot have a high probability result here. \\

\noindent
\subsection{Oracles for Convex Sets} 
The seminal work of~\cite{GLS1988} showed that the Ellipsoid Algorithm can be viewed as a reduction of an Optimization Oracle 
to a Separation Oracle for convex sets. Since then, there has been an extensive study of reductions among approximate oracles (also referred to as ``weak'' oracles).
There are two important parameters in a reduction from an oracle ${\cal A}$ with error $\delta$ to an oracle ${\cal B}$ with error $\varepsilon$ --  the number of calls to ${\cal B}$ used and the increase in error, namely, $\delta/\varepsilon$. Of these, the number of oracle calls has received much attention
since it is a measure of running time.
But the error parameter has also been taken into account in most results starting with ~\cite{GLS1988}. The best known bounds on the number of oracle calls in reductions are due to~\cite{LeeSV18}; they achieve near-linear number of oracle calls. The known error increase factors are $\Omega(d)$, where $d$ is the dimension. Our proof of $\rsh$ gives a 
%A proof of inequality \ref{SHT-Q}) will give us a 
simple polynomial time reduction (for fixed $\delta$) from Separation to Optimization
with the error increase factor of $O^*(\sqrt d)$ which we also show is best possible within log factors.

We define our approximate oracles (which differ from the traditional definitions of approximate oracles --  see \Cref{approx-oracles}) and state our nearly-matching upper and lower bounds for error increase in the reduction. Here, $B_d$ denotes the unit ball in $\Re^d$. 

\begin{defn}\label{defn-Sep}
For a non-empty convex set $K\subseteq\Re^d$ and $\delta\in (0,1)$, a separation oracle for $K$ with error $\delta$, denoted $\sep_\delta(K)$ oracle,  takes as input any $a\in \Re^d$ and returns a valid option among the two below (Note: Both may be valid):
\begin{itemize}
\item $a\in K+\delta \Delta (K) B_d$.
\item Returns $u$ satisfying
$u\cdot a >\mbox{Max}_{y\in K}u\cdot y$
\end{itemize} 
%Let $\sep_\delta(K)$ denote the set of separation oracles with error $\delta$ for $K$. 
\end{defn}

\begin{defn}\label{olp}
For a non-empty convex set $K\subseteq \Re^d$ and $\varepsilon\in (0,1)$, an Optimization oracle for $K$ with error $\varepsilon$, denoted $\opto_\varepsilon(K)$ oracle,  takes as input any $u\in \Re^d, |u|=1$, and returns a point $x(u)$ satisfying {\em both} these conditions :
\begin{itemize}
\item $x(u)\in K+\varepsilon\Delta(K)B_d$, and
\item $u\cdot x(u)\geq \mbox{Max}_{y\in K}u\cdot y-\varepsilon\Delta(K)$
\end{itemize}
%Let $\opto_\varepsilon(K)$ denote the set of separation oracles with error $\varepsilon$ for $K$. 
\end{defn}

\begin{remark}\label{approx-oracles}
Starting with ~\cite{GLS1988}, the second option in the traditional definition of approximate oracles usually replaces the $K$ we have in that option with a subset of $K$, namely, the subset of all points with the property that a ball of specified size centered at that point is wholly inside $K$. This is necessary since the proof of convergence of the Ellipsoid algorithm is by shrinking the volume of the ellipsoid containing $K$.
If $K$ is not full dimensional, the subset of $K$ is empty and a worst-case (adversarial) oracle can always return this option giving us no information on $K$. Our definition makes the stronger assumption  with $K$ in the second option, thus, dealing with (among other examples) non-full dimensional $K$.
\end{remark}

\begin{restatable}{theorem}{sepopt}
\label{thm-sep-to-opt}
Let $\delta \in (0,1)$ be any constant and $K$ be  a polytope in $\Re^d$ with $k\geq 1$ vertices  given by an  $ \opto_\varepsilon(K)$ oracle $\cal A$, with $\varepsilon\leq \delta/100 \sqrt{d}$. Then there is an $\sep_\delta(K)$ oracle which is obtained by making  poly$_\delta (dk)$ calls to $\cal A$. 
% and $K$ is 
% given by an optimization oracle with error $\delta/100\sqrt d$. Then, there is a Separation oracle for $K$ with error $\delta$ which makes poly$_\delta (dk)$ calls to the optimization oracle. 
\end{restatable}

It is worth noting that the  calls to optimization oracle in the above result use random vectors $u$, and so the reduction algorithm is randomized. The following result shows that the condition $\varepsilon\leq \delta/100 \sqrt{d}$ above is almost tight in the sense no deterministic algorithm can beat it by more than $\log$ factors. 

\begin{theorem}\label{thm-sep-opt-lower}
There is no deterministic polynomial time reduction from an  $\sep_{O(1)}(K)$ oracle to an oracle in $\opto_{\Omega(\log d/\sqrt{d})}(K).$ 
% Separation oracle of error $1/10$ to an Optimization oracle of error $\Omega(\log d/d)$.
\end{theorem}

\subsection{Algorithmic application of $\rsh$}\label{algorithmic-rsh}
We now discuss the second main contribution of our work, i.e., applications of $\rsh$ to learning vertices of a latent polytope.

\vspace*{0.1 in}
\noindent
{\bf Problem Formulation.}
Several latent variable problems including Clustering, LDA, MMBM can be reduced 
(See~\Cref{sec:kolp} for details) to a problem that we call $\kolp$: given an $\varepsilon$-optimization oracle for a $k$ vertex polytope $K\subseteq {\bf R}^d$, learn the vertices of $K$ (approximately). 
We define two simpler (than $\kolp$) problems -- $\ListLearn$ and $\Hausdorf$, that are related to $\kolp$, and then we define $\kolp$. 
%{\bf Notation} For a polytope $K$, we denote by $\Delta(K)$ its diameter (the maximum distance between two points in $K$).
%Often $K$ is clear from the context and we will use $\Delta$ for $\Delta(K)$.

% \begin{defn}[Optimization Oracle]\label{OLP}
% Let $K$ be a polytope in $\Re^d$ with diameter $\Delta(K)$.  An {\em $\varepsilon$-optimization} oracle for $K$ takes as input  a unit length vector $u\in \Re^d$ and returns a vector $x\in \Re^d$ satisfying 
% $$x\in K+\varepsilon \Delta(K) B_d;\quad u\cdot x\geq \mbox{Max}_{y\in K} \, u\cdot y-\varepsilon \Delta(K),$$ 
% where $B_d$ denotes the unit ball in $\Re^d$. 
% \end{defn}

The first problem, $\Hausdorf$, seeks to find approximation to a polytope $K$ when we are given an approximate optimization oracle for the polytope. 
\begin{defn}[$(\varepsilon,\delta)$-$\Hausdorf$-Problem]
\label{Hausdorf-problem}
Given an  $\opto_\varepsilon(K)$ oracle for a polytope $K$ in ${\bf R}^d$ with $k$ vertices, find a set $P$ of $m=\poly_\delta(dk)$ points such that $\Haus(CH(P),K)\leq \delta\Delta(K),$
where, $\Haus$ denotes Hausdorff distance (see~\Cref{defn:Haus} for a formal definition).    
\end{defn}

In the problem $\ListLearn$, we also wish to find a small list of points, such that each vertex of $K$ is close to at least one point in this list. 
\begin{defn}\label{LL}[$(\varepsilon,\delta)$-$\ListLearn$ Problem]
Given an $\opto_\varepsilon(K)$ oracle for a polytope $K$ in $\Re^d$ with $k$ vertices, each separated from the convex hull of the other $k-1$ vertices by at least $\delta\Delta(K)$, find a list $P\subseteq K+\delta \Delta(K) B_d$ of $m=\poly_\delta(dk)$ points 
such that for every vertex $v$ of $K$, there is some $v'\in P$ with $|v-v'|\leq\delta\Delta(K)/10.$
\end{defn}

When the parameters $\varepsilon, \delta$ will be clear from the context, we shall abbreviate the above two problems as $\Hausdorf$ and $\ListLearn$ problems respectively. It is not difficult to see that any solution $P$ to $\ListLearn$ is also a solution to $\Hausdorf$, but, the converse need not hold: CH$(P)$ may nearly contain $K$ without $P$ having any point close to some vertex of $K$. 
%
 % An $\varepsilon$-optimization oracle ${\cal O}$, when queried on a vector $u\in {\bf R}^d,$ returns an approximately feasible point $x$ which is approximately optimal, i.e., $x$ is within distance $\varepsilon\Delta$ of $K$ and satisfies  $u\cdot x\geq \mbox{Max}_{y\in K} u\cdot y-\varepsilon\Delta.$
Our technical results (see below for the informal versions)
%We will state below theorems giving (a simplified version of) our 
%results, which say essentially that if 
show that if $\varepsilon\in O_\delta(1/\sqrt d)$ \footnote{$O_\delta(x)$ stands for $f(\delta)x$ for some function $f$.}, then we can solve the above-mentioned problems efficiently  and indeed then, the following simple algorithm gives the desired answers (the proof crucially uses $\rsh$): 
\begin{framed}
{\bf  Random Probes Algorithm}
\begin{itemize}
    \item Pick uniformly at random unit vectors $u_1,u_2,\ldots u_m$, where $m=$poly$_\delta(dk)$.
    \item Return $P$, which is the set of $m$ answers of the $\opto_\varepsilon(K)$ oracle to the queries $u_1,u_2,\ldots u_m$. 
\end{itemize}
\end{framed}
The first result (see~\Cref{upper-bound-3} for a formal statement) states that the convex hull of answers to polynomially many random queries to the approximate optimization oracle approximates $K$ well. 

\begin{theorem} 
[Hausdorff Approximation from oracle (Informal)]
\label{hausdorf-0}
Consider an instance of the $(\varepsilon, \delta)$-$\Hausdorf$ problem for a polytope $K \subseteq \Re^d$. Assume  that
%Suppose in the $\Hausdorf$ problem  
$\varepsilon\in O_\delta(1/\sqrt d)$, and let $P$ be the set of points returned by the {\bf Random Probes Algorithm} above. 
%Let $u_1,u_2,\ldots ,u_m,$ $m=$poly$_\delta (dk),$ be random vectors in ${\bf R}^d$
%and 
%$P$ be the set of $m$ answers from the $\varepsilon$-optimization oracle on queries $u_1,u_2,\ldots ,u_m$. 
Then with high probability, 
$$ \Haus(CH(P),K)\leq \delta\Delta(K).$$
\end{theorem}

The second result (see~\Cref{upper-bound-2} for a formal statement) shows that as long as each vertex of $K$ is {\em well-separated} from the convex hull of the other vertices of $K$, the set $P$ constructed by the {\bf Random Probes Algorithm} contains an approximation to each of the vertices. Thus, answers to polynomially many random queries list-learns the polytope.  

%This follows from Theorem (\ref{upper-bound-3}).  

\begin{theorem}[List-Learning from Oracle (Informal)]\label{list-0}
Consider an instance of  the $\ListLearn$ problem for a polytope $K \subseteq \Re^d$, and assume that  
  $\varepsilon\in O_\delta(1/\sqrt d)$. Then, with high probability, the set $P$ output by the {\bf Random Probes  Algorithm} above has the following property:
%   Let $u_1,u_2,\ldots ,u_m,$ $m=$poly$_\delta (dk)$ be random vectors in ${\bf R}^d$
% and 
% $P$ be the set of $m$ answers from the oracle  on queries $u_1,u_2,\ldots ,u_m$.
% Then, with high probability, 
for every vertex $a$ of $K$, there is a point $a'\in P$ with
$$|a'-a|\leq \delta\Delta(K)/10.$$
\end{theorem}
%This follows from Theorem (\ref{upper-bound-2}).

The $\sqrt d$ factor in both the above theorems is near-optimal (within a $\log d$ factor). Indeed, we shall prove:

\begin{restatable}[Oracle Lower Bound]{theorem}{lowerbound}
\label{lower-bound-0}
The problem where, one is required to output a point which is within $\Delta (K)/10$ of some vertex of $K$, given only by an $\opto_\varepsilon(K)$ oracle,  cannot be solved in deterministic polynomial time when $\varepsilon\geq 8\ln d/\sqrt d$.
\end{restatable}

We now define the $\kolp$ problem (the parameters $\varepsilon, \delta$ in the definition will often be clear from the context and may not be mentioned explicitly). The problem statement is similar to that of $\ListLearn$, but we want to output a list of exactly $k$ points. 
\begin{defn}\label{OLP}[($\varepsilon,\delta)$-$\kolp$ Problem]
Under the same hypothesis as for the $\ListLearn$ problem, find  a set of points $P$,  $|P|=k$, satisfying the following condition: for each vertex $v$ of $K$, there is a (unique) point $ v'\in P$ such that  $|v-v'|\leq\delta\Delta(K)/10$.
\end{defn}

Our next result gives a strengthening of~\Cref{list-0}. 
\begin{theorem}\label{kolp-0}
Consider an instance of the $\kolp$ problem on a polytope $K \subseteq \Re^d$, and assume 
$\varepsilon \in O_\delta(1/\sqrt d)$ in a $\kolp$ problem. Let $P$ be the set of points returned by the {\bf Random Probes Algorithm}. Then, in polynomial time, we an find a $Q\subseteq P, |Q|=k$ satisfying the following condition:  for each vertex $v$ of $K$,  there is a (unique) $v'\in P, |v-v'|\leq\delta\Delta(K)/10$.
\end{theorem}

The algorithm for finding $Q$ from $P$ is likely of independent interest. We call this problem the ``Soft Convex Hull'' problem and it is described in~\Cref{sec:soft}.

\paragraph{Do Approximate Optimization Oracles exist?}
\label{Data-Gen-1}
The answer to this question is a qualified Yes. They exist, but unfortunately, as we point out below,  for many latent variable problems including the simple mixture of two Gaussians with means separated by $\Omega(1)$ standard deviations, we do not get $\varepsilon\in O_\delta(1/\sqrt d)$, when, $k<d$. Thus,  we do not satisfy the hypothesis of the results mentioned in~\Cref{hausdorf-0},\Cref{list-0}, and \Cref{kolp-0}.  
But we are able to tackle this hurdle by projecting to the $k$-SVD subspace (of the input data points which satisfy conditions discussed below) where, we do get the necessary 
$\varepsilon\in O_\delta(1/\sqrt k)$. 

First we observe that approximate optimization oracles arise in a 
natural setting -- that of latent variable models.~\cite{BhattacharyyaK20} show that these models can be reduced to a geometric problem called $\lkp$ described below. [We will not reproduce the reduction here.] 
%It will be simple to formulate the oracle problem $\kolp$ from $\lkp$. 
%   
$\lkp$ is the following problem: Let $K$ be a $k$ vertex polytope in $\Re^d$. 
Let $\Mdot{1}, \ldots, \Mdot{k}$ denote the vertices of $K$.  
Assume that there are latent (hidden) points $P_{\cdot ,j}, j=1,2,\ldots ,n,$ in $K$. The observed data points $A_{\cdot,j}, j=1,2,\ldots ,n$ are generated (not necessarily under any stochastic assumptions) by adding {\it displacements} $A_{\cdot,j}-P_{\cdot,j}$ respectively to $P_{\cdot,j}$. Let\footnote{By the standard definition of spectral norm, it is easy to see that $\sigma_0^2$ is the maximum mean squared displacement in any direction.}
$$\sigma_0 :=\frac{||\bP-\bA||}{\sqrt n}.$$
We assume that there is a certain $w_0$ fraction of latent points close to every vertex of $K$, i.e., for all $\ell \in [k]$, 
% ???Add a line on why $w_0$ is necessary??? \aknote{Right now, $w_0$ is a free parameter, so not sure what the lower bound statement will be. }
% {\color{red} Not sure I understand the above comment... You mean the lower bound on card of C ell? - ravi}

$$  C_\ell :=\{j:|P_{\cdot,j}-M_{\cdot,\ell}|\leq\frac{\sigma_0}{\sqrt{w_0}}\} \mbox{ satisfies    }
|C_\ell|\geq w_0n.$$

\begin{theorem}[From Data to Oracles] \label{subset-smoothing-0}
%[{\bf From Data to Oracles}]
Using the above notation, the following ``Subset Smoothing  algorithm'' gives us a polynomial time  $\opto_{\frac{4\sigma_0}{\Delta \sqrt{w_0}}}(K)$ oracle..
\end{theorem}

\begin{framed}

{\bf Subset Smoothing Algorithm}

Given query $u$, let $S$ be the set  of the $w_0 n$ $j$'s with the highest $u\cdot A_{\cdot,j}$ values. 

Return 
$A_{\cdot , S}:=\frac{1}{w_0n}\sum_{j\in S} A_{\cdot,j}.$

\end{framed}

The Subset Smoothing algorithm was used in~\cite{BhattacharyyaK20}. It is also reminiscent of Superquantiles~\cite{rockafellar2014random}, though our use here is not directly related to them.
While this theorem helps us get optimization oracles, the error guarantee of $O(\sigma_0/\Delta \sqrt{w_0})$ is not good enough in many applications. An elementary example illustrates this issue:

Consider a mixture of two equal weight standard Gaussians centered at $-v$ and $v$, where, $v$ is a vector of length 10. [This fits the paradigm ``means separated by $\Omega(1)$ standard deviations''.] Then, data generated by the mixture model fits our data generation process with $K=\{ \lambda v, \lambda\in [-1,1]\},$ and each $\Pdot{j}$ is either $v$ or $-v$ depending on the Gaussian from which the point has been sampled. Here $\Adot{j}$ denotes the actual sampled point from the mixture. 
Now, $\Delta =20$ and it can be seen from Random Matrix Theorems (see e.g., \cite{vers}) that $\sigma_0=O(1)$ with high probability. 
So, $\sigma_0/\Delta\sqrt{w_0}\in O(1)$ with high probability, and hence, the Theorem above 
yields an $\opto_\varepsilon(K)$ oracle with $\varepsilon\in\Omega(1)$. But
$d$ can be arbitrarily large and so we do not have the required $\opto_{O(1/\sqrt d)}(K)$  oracle. 

This elementary example can be tackled in several ways. Our algorithm which we call the ``$\kolp$ algorithm'' is simply stated and works in general settings (including on this toy example) for several Latent Variable problems (see~\Cref{sec:kolp} for details). The main idea is to first project the input points on a suitable SVD subspace and then use the approximate optimization oracle in the projection.

\paragraph{SVD and the $\kolp$ Algorithm}

We now state the result (see~\Cref{LkP-Main-Thm} for a formal version) for $\kolp$ in the setting of $\lkp$. As mentioned above, this uses SVD followed by subset smoothing. 

\begin{theorem}\label{kolp-thm}
Recall the notation and assumptions of~\Cref{subset-smoothing-0}. In addition, we assume that each vertex of $K$ is $\delta\Delta(K)$ far from the convex hull of other vertices of $K$, where, $\delta$ satisfies:
$$\sigma_0\leq c\delta^2\Delta\sqrt {w_0}/\sqrt k.$$
Then, the set of points $P$ returned by the following $\kolp$ algorithm list learns the vertices of $K$. 
Further, we can find a subset $Q$ of $P$ with $|Q|=k$ and for each $v$, vertex of $K$, $Q$ contains a $v'$ with $|v-v'|\leq \delta\Delta /10$ :
\end{theorem}

\begin{framed}

{\bf Algorithm $\kolp$}

\begin{enumerate}
\item Project to the $k$-dim SVD subspace $V$ corresponding to the points $\Adot{j}$.
\item Pick $m=$ poly$_\delta(k)$ random vectors $u_1,u_2,\ldots ,u_m$ in $V$.
\item For each $u_i$, take the mean of the $A_{\cdot,j}$ with the $w_0n/2$ highest values of $u_i\cdot A_{\cdot,j}$. 
\item Let $P$ be the  set of $m$ means computed in the step above.
\item Output a subset $Q$ of $P$, $|Q| = k$, using~\Cref{kolp-0}. 
\end{enumerate}
\end{framed}

%The theorem follows from Theorem (\ref{LkP-Main-Thm}). 
We sketch here the steps in the proof (the details are contained in the proof of~\Cref{LkP-Main-Thm}.) 

\begin{proof} [Sketch]
Let $\hatK$ denote projection of $K$ onto $V$. By~\Cref{subset-smoothing-0}   , Step 3 of Algorithm $\kolp$ is an    $\opto_\varepsilon(\hatK)$ oracle, where $\varepsilon = {O\left(\frac{\sigma_0}{\Delta \sqrt{w_0}}\right)}$. Also, each $\hMdot{\ell}$, which is the projection of $\hMdot{\ell}$ on $V$, is $O(\delta\Delta(K))$ far from the convex hull of the other vertices of $\hatK$ (See \Cref{lem:close}.) Now, \Cref{kolp-0} applied  to $\hatK$ in the subspace $V$ implies the desired result. 
\end{proof}

It is worth noting that data obtained from several generative models are known to satisfy the $\lkp$ condition stated in~\Cref{kolp-thm}, e.g.,  Stochastic Mixture models with $k$ components, Topic Models, Mixed membership community models. 
\paragraph{From List Learning to $\kolp$}
%
% {\color{red} I am in two minds about this section. On the one hand, as stated below, there are new ideas. On the other hand, they are quite different from the flow of the paper and also there is literature including the paper of Avrim Blum's and perhaps other Geoometry results (someone in my talk mentioned a paper of Mitchell of Cornell) which we need to describe...
%
% So, either we worry about all these matters or we leave this aside, from the intro and make Soft Convex Hulls a subsection as suggested by Amit... I am more in favour of leaving this aside...}
%
As outlined above, the $\kolp$ algorithm works in two stages: (i) Project the data points on the SVD subspace $V$ of dimension $k$, and (ii) make polynomially calls to the $\opto_\varepsilon(K)$ oracle, where each query is given by a randomly chosen unit vector in the subspace $V$ (as in the statement of~\Cref{hausdorf-0}) -- let $P$ be the set of points returned by the oracle. The first statement in~\Cref{kolp-thm} shows that the convex hull of $P$ is close to $K$.

Obtaining approximations to the vertices of $K$ from $P$ requires addressing a new problem: given a set of points $W$, find a small subset $T$ of $W$, such that their convex hulls are close. We call this the {\em soft convex hull} problem. A similar problem was addressed by~\cite{BlumHR19}; however they gave a bi-criteria approximation algorithm for this problem. Under stronger assumptions, where we assume that there in the optimal solution $T^\star$, each point of $T^\star$ is well-separated from the convex hull of the rest of the points of $T^\star$, we show that one can recover approximations to each of the points in $T^\star$. Applying this result to the set of points $P$ returned by the optimization oracle, we get a set of $k$ points $Q$, each of which approximates a unique vertex of the polytope $K$. 

The algorithm for obtaining soft convex hull proceeds as follows. We first prune points $w \in W$ which have the following property: consider the subset $X$ of points in $W$ which are sufficiently far from $w$. Then $w$ is close the convex hull of $X$. After pruning such points from $W$, we pick a subset of points which are sufficiently far-apart from each other. The main technical result shows that this procedure outputs the desired set $T$.

\subsection{Related Work}
The seminal work of \cite{GLS1988} showed that optimization of a convex function over a convex set can be reduced to separation oracles using the classical ellipsoid algorithm. There has been active research on  reducing the number of separation (or membership) oracle queries (see e.g.~\cite{abernethy2015faster,KalaiV06, LovaszV06}).~\cite{LeeSV18} showed that, for a ``well-rounded'' convex set $K$ (i.e.,  $K$ has unit diameter and contains a ball of constant radius), we can get a separation oracle by making ${\widetilde O}(d)$ calls to an optimization oracle for $K$. In this reduction, in order to get a separation oracle with error $\delta,$ they use an optimization oracle with error $\poly(\delta/d)$. As mentioned in the previous section, $\rsh$ Theorem implies that we can obtain a separation oracle with error $\delta$ for a convex polytope $K$ from an optimization oracle with error $O(\delta/\sqrt{d})$. 

The well known result of~\cite{jl} shows that given a set of $n$ points in $\Re^d$, projection to a random subspace of dimension $O(\log n/\varepsilon^2)$ preserves all pair-wise distances up to $(1 + \varepsilon)$-factor with high probability. Further, this bound on the dimension on which the points are projected is known to be tight~\cite{alon2003problems, LarsenN17}. Note that in our setting, there are $k+1$ points ``of interest'', namely, the $k$ vertices of $K$ and a point $a\notin K$ and by the above, a random projection to $O^*(\log k)$ dimensional space preserves all pairwise distances among them. But this is not sufficient for our problems. We need separation of $a$ from all of $K$ in the projection. We achieve this by projecting to a set of random 1-dimensional subspaces, and show that the distance between a point and a polytope does not scale down by more than $O(\sqrt{d})$ factor for at least one of them with high probability (this is an immediate corollary of the $\rsh$ Theorem).

The problem of learning vertices of a polytope arises in many settings where data is assumed to be generated by a stochastic process parameterized by a model. Examples include topic models~\cite{Blei12}, stochastic block models~\cite{AiroldiBEF14}, latent Dirichlet allocation~\cite{BleiNJ03}. A variety of techniques have been developed for these specific problems (see e.g.~\cite{AnandkumarFHKL15, AroraGHMMSWZ18,HopkinsS17}).~\cite{BhattacharyyaK20} (see also~\cite{BakshiBKWZ21}) proposed the latent $k$-polytope (LkP) model which seeks to unify all of these latent variable models. In this model, there is {\em latent} polytope with $k$ vertices, and data is generated in a two step process: first we pick {\em latent} points from this polytope, and then the observed points are obtained by perturbing these latent points in an adversarial manner. They showed that under suitable assumptions on this deterministic setting, one can capture the above-mentioned latent variable problems. Assuming strong separability conditions on the vertices of the polytope (i.e., each vertex of $K$ is far from the {\em affine} hull of other vertices of $K$), they showed that one can efficiently  recover good approximations to the vertices of the polytope from the input data points. In comparison, our assumption on $K$ is that each vertex of $K$ is far from the convex hull of the remaining vertices of $K$. This is a much milder condition, e.g., it allows a polytope with more than 2 vertices in a plane.~\cite{BhattacharyyaK021} showed how to infer the parameter $k$ from data in the LkP setting (under the strong separation condition). 

~\cite{BlumHR19} addressed a problem similar to the $\Hausdorf$ problem: instead of an $\varepsilon$-optimization oracle for a polytope $K$, we are given an explicit set $P$ of points, whose convex hull is within  Hausdorff distance at most $\delta\Delta(K)$ from $K$. They are able to get better dependencies on the parameters $k,\delta,\varepsilon$ than \Cref{upper-bound-3} under these stronger assumptions. 

\section{Preliminaries}
 For two points $x, y \in \Re^d$, $|x-y|$ denotes the Euclidean distance between the points. Given a point $x \in \Re^d$ and a subset $X \subseteq \Re^d$, define $\dist(x, X)$ as the minimum distance between $x$ and a point in $X$, i.e., $\inf_{y \in X} |x-y|$. 
For a set of points $X$, $\dia(X)$ denotes the diameter of $X$, i.e., $\sup_{x, y \in X} |x-y|$. We  denote the convex hull of $X$ by $\ch(X)$. For two subsets $A, B$ of $\Re^d$, define their Minkowski sum $A+B$ as $\{x+y: x \in A, y \in B\}$. Similarly, define $\lambda A$, where $\lambda \in \Re$, as $\{ \lambda x: x \in A\}$.
For an $m \times n$ matrix $B$, we use $B_{\cdot, j}$ to denote the $j^{th}$ column of $B$. For a subset $S \subseteq [n]$ of columns of $B$, $B_{\cdot, S}$  denotes $\frac{1}{|S|} \sum_{j \in S} B_{\cdot, j}$. Often, we represent the vertices of a polytope $K$ in $\Re^d$ by a $d \times k$ matrix $M$, and so the columns $\Mdot{1}, \ldots, \Mdot{k}$ would represent the vertices of $K$. We shall use the notation $\poly_\delta(z)$ to denote a quantity which is $z^{\poly(1/\delta)}$. Further the notation $O_\delta(z)$ shall denote a quantity which is $f(\delta) z$, where $f(\delta)$ is a function depending on $\delta$ only (and hence, is constant if $\delta$ is constant). 

We now give an outline of rest of the paper. In~\Cref{sec:rsh}, we prove the Random Separating Hyperplane theorem.  In~\Cref{sec:haus}, we prove~\Cref{hausdorf-0} by showing that an $\opto_\varepsilon(K)$ oracle leads to efficient constructions of approximation to $K$. In~\Cref{sec:listlearn} we give an algorithm for the $\ListLearn$ problem under the stronger assumption that the vertices of $K$ are well separated. We also prove the lower bound result~\Cref{lower-bound-0} in this section. In~\Cref{sec:kolp}, we extend the algorithm for $\ListLearn$ to the $\kolp$ problem. This requires the concept of soft convex hulls. The algorithm for constructing soft convex hulls in given in~\Cref{sec:soft}. Finally, in~\Cref{sec:lkp}, we apply the $\kolp$ algorithm for the latent polytope problem. As note earlier, in the setting of latent polytopes $K$, we can only guarantee $\opto_\varepsilon(K)$ oracles with $\varepsilon$ being $O(1/\sqrt{k})$, whereas our algorithm for $\kolp$ requires $\varepsilon$ to be $O(1/\sqrt{d})$. We handle this issue by projecting to a suitable SVD subspace and executing the $\kolp$ algorithm in this subspace. We conclude with some open problems in~\Cref{sec:open}.

\section{Reduction from Separation to Optimization Oracles}\label{section-oracles}

In this section, we prove~\Cref{thm-sep-to-opt}, which is reproduced here. 

% \begin{defn}\label{defn-Sep}
% For a non-empty convex set $K\subseteq\Re^d$ and $\delta\in (0,1)$, a separation oracle for $K$ with accuracy $\delta$ takes as input any $a\in \Re^d$ and returns a valid option among the two below (Note: Both may be valid):
% \begin{itemize}
% \item $a\in K+\delta \Delta (K) B_d$.
% \item Returns $u$ satisfying
% $u\cdot a >\mbox{Max}_{y\in K}u\cdot y$
% \end{itemize} 
% \end{defn}

% \begin{defn}\label{defn-Opt}
% For a non-empty convex set $K\subseteq \Re^d$ and $\varepsilon\in (0,1)$, an Optimization oracle for $K$ with accuracy $\varepsilon$ takes as input any $u\in \Re^d, |u|=1$ and returns a point $x(u)$ satisfying :
% \begin{itemize}
% \item $x(u)\in K+\varepsilon\Delta(K)B_d$,
% \item And $u\cdot x(u)\geq \mbox{Max}_{y\in K}u\cdot y-\varepsilon\Delta(K)$
% \end{itemize}
% \end{defn}

% \begin{theorem}\label{thm-sep-to-opt}
% Suppose $K$ is a polytope in $\Re^d$ with $k\geq 1$ vertices and $\delta\in (0,1)$ and $K$ is 
% given only by an optimization oracle with accuracy $\delta/100\sqrt d$. Then, there is a Separation oracle for $K$ with accuracy $\delta$ which makes poly$_\delta (dk)$ calls to the optimization oracle. 
% \end{theorem}

% {\bf Remark} The calls are with random $u$'s, so the reduction algorithm is randomized.
\sepopt*
\begin{proof}
Consider a point $a \in \Re^d$. 
Use the  oracle $\cal A$ on $\poly_\delta(kd)$ random vectors $u, |u|=1$. Let $U$ denote the set of these unit vectors. If $a\notin K+\delta\Delta(K) B_d,$ then, by $\rsh$, with high probability, there is an $u \in U$ such that $u\cdot a> \mbox{Max}_{y\in K}u\cdot y+\delta\Delta(K)/(10\sqrt d)$. For this $u$, we have (using 
$x(u)\in K+(\delta\Delta(K)/100\sqrt d)B_d$):
\begin{equation}\label{1235}
u\cdot a>u\cdot x(u) +(\delta \Delta(K)/11\sqrt d). \end{equation}
Conversely, for any $u \in U$ satisfying (\ref{1235}),
we have (using $u\cdot x(u)\geq \mbox{Max}_{y\in K}u\cdot y-\delta\Delta(K) /100\sqrt d$):
$u\cdot a> \mbox{Max}_{y\in K}u\cdot y +
\delta\Delta(K) /20\sqrt d$.
Thus, we obtain a $\sep_\delta(K)$ oracle as follows: check if (\ref{1235}) holds for any $u \in U$. We are guaranteed to find one with high probability if $a\notin K+\delta\Delta B$ and that provides the required separating hyperplane.
\end{proof}

The proof of the lower bound result~\Cref{thm-sep-opt-lower} is very similar to that of~\Cref{lower-bound-0}, which is given in~\Cref{sec:listlearn}. 

% Now we prove that the loss in accuracy of a factor of $O(\sqrt d)$ in the reduction of Theorem 
% (\ref{thm-sep-to-opt}) is nearly optimal (within logarithmic factors). There is substantial literature on reductions between these oracles, improving the number of calls to the Optimization oracle the reduction makes. A nearly optimal dependence of this on $d$, namely $O^*(d)$ has been obtained by ??? Lee, Sidford and Vempala??. However all reductions in the literature to our knowledge have accuracy loss factor of $\Omega(d)$.

% \begin{proof}  (Of Theorem (\ref{thm-sep-opt-lower})): 
% The proof is producing an ``adversarial'' example. $K$ will be of the form $\{ x: x=\lambda u, \lambda\in [-1,1]\}$ and $a$ will be a unit length vector perpendicular to $u$. The choice of $u,a$ is adversarialy made to contraditct any result the reduction purports to produce. 

% Note
% $a\notin K+\delta \Delta(K)B_d$. Assume a separation oracle of accuracy $1/10$ exists and it asks queries $v_1,v_2,\ldots ,v_q$ to the Optimization oracle an returns $v_0$ with the assertion that 
% $v_0\cdot a>\mbox{Max}_{y\in K}v_0\cdot y$. The opt oracle returns 0 on all queries.
% It can be shown that there are vectors $u,a$ satisfying 
% $|u\cdot v_i|\leq 4\ln d/\sqrt d,$ and $|a\cdot v_i|\leq 4\ln d/\sqrt d $ for $i=0,1,\ldots ,q$, 
% so 0 is a consistnt with this $u$ definign the latent $K$ and also we have
% $\mbox{Max}_{y\in K}v_0y\geq 0$. 
% Wlg assume $v_0\cdot a \leq 0$ (if not replace $a$ by $-a$.)
% Clearly, $v_0$ does not separate prodcuing a contradiction.
% \end{proof}

\section{Random Separating Hyperplane ($\rsh$) Theorem}
\label{sec:rsh}
In this section, we prove $\rsh$ for polytopes: If a point $p$  point is at a distance from a polytope $K$, then, the Gaussian measure of the set of {\em well-separating} hyperplanes has a positive lower bound depending on the number of vertices in $K$. More specifically, we show:

\begin{theorem}
\label{thm:sephp}
Suppose $K$ is a polytope in $\Re^d$ with $k$ vertices and diameter $\dia(K)$. Suppose  $a$ is a point in ${\bf R}^d$ and $\delta\in (0,1]$ with
\begin{align}
\label{eq:condhp}
\min_{y\in K} |a-y|\geq\delta \dia(K).
\end{align}
Let $V$ be an $m$-dimensional subspace containing $\lspan(K\cup \{ a\})$ and let $u$ be a random vector drawn from the normal distribution $N(0, I_m)$ in $V$. 
Then,
$$\prob_u\left[ (u\cdot a-\max_{y\in K}u\cdot y) \geq  |u| \cdot \delta \dia(K) \cdot \frac{\sqrt{\log k}}{\sqrt{\log k} + 4 \delta \sqrt{m}} \right]\geq
            \frac{1}{40}    k^{-10/\delta^2}.$$
\end{theorem}          

% \noindent
% {\bf Remark:}  Such a theorem does not hold for general convex sets.
% For example,  consider the Euclidean space $\Re^d$ and the $d-2$ dimensional sphere $K:=\{x \in \Re^d: |x|\leq 1, x_1=0\}$. Observe that the diameter $\dia(K)$ of $K$ is 1. 
% Let $a$ denote the point $(1,0,0,\ldots ,0)$. It is at distance 1 from $K$, and so the condition~\eqref{eq:condhp} is satisfied if we set $\rho$ to $1/2$.  But, with high probability,  a random unit length vector $u$ has $u\cdot a\approx \frac{c}{\sqrt d}$ for some constant $c$,  whereas the maximum of  $v\cdot x$ over all points $x$ in the sphere $X$ is about 1. Therefore the probability of the event in the statement of~\Cref{thm:sephp} is exponentially small in $d$. 

\begin{proof}
 Let $\zeta_1,\zeta_2,\ldots ,\zeta_k$ be the vertices of $K$. Let $b$ be the closest point in $K$ to $a$, and define 
$$w=\frac{a-b}{|a-b|}.$$
Then by standard Convex Geoemtry arguments, we have for all points $y \in K$:
\begin{equation}\label{101}
w\cdot y\leq w\cdot b.
\end{equation}
Let $u$ be as in the statement of the theorem. 
We can write $u$ as 
$$u=\lambda w+z, \mbox{ where }z\perp w.$$

In order to prove the theorem, we first express $(u \cdot a  - \max_{y \in K} u \cdot y)$ as 
\begin{align}
\label{eq:step1}
u \cdot (a-b) - \max_{y \in K} u \cdot (y-b)  =  u \cdot (a-b) - \max_{\ell = 1, \ldots, k} u \cdot (\zeta_\ell-b)
\end{align}

It suffices to show that with reasonably high probability, there is a lower bound on $u \cdot (a-b)$ and an upper bound on $u \cdot (\zeta_\ell-b)$ for all $\ell$. Observe that 
\begin{align}
\label{eq:step2}
u \cdot (\zeta_\ell-b) = \lambda w \cdot (\zeta_\ell-b) + z \cdot (\zeta_\ell-b) \stackrel{\eqref{101}}{\leq} z \cdot (\zeta_\ell-b).
\end{align}
Therefore, it suffices to upper bound $z \cdot (\zeta_\ell - b)$ for all vertices $\zeta_\ell$. We now define a sequence of events and bound the probability of each of these events. 

\begin{fact}
\label{fact:1}
Let $\cE_0$ be the event $[|z| \leq 4 \sqrt{m}]$. Then $\Pr[\neg \cE_0] \leq 1/20. $
\end{fact}
\begin{proof}
The coordinates of $z$ are independent $N(0,1)$ random variables, and since $z$ lies in an $(m-1)$-dimensional subspace, the result follows from standard concentration bounds for sum of squares of Gaussian random variables. 
\end{proof}

\begin{fact}
\label{fact:2}
For $\ell=1,2,\ldots , k$, define the  event
${\cal E}_\ell$ as $[\left| z\cdot (\zeta_\ell-b)\right|\leq 2\sqrt{\ln k} \dia(K)].$ Then $\Pr[\neg \cE_\ell] \leq \frac{1}{4k^2}.$
\end{fact}
\begin{proof}
Observe that $z \cdot (\zeta_\ell - b)$ is a 1-dimensional Gaussian random variable with zero mean and  with variance $|\zeta_\ell - b|^2 \leq \dia(K)^2$. The result now follows the tail bounds for a 1-dimensional Gaussian random variable. 
\end{proof}

\begin{fact}
\label{fact:3}
Let
${\cal E}_{k+1}$ denote the event that $\lambda\geq 3\sqrt{\ln k}/\delta.$ Then $\Pr[\neg \cE_{k+1}] \geq  \frac{1}{10} k^{-10/\delta^2}.$
\end{fact}
\begin{proof}
Since $\lambda$ is a 1-dimensional Gaussian random variable, the result follows from tail bounds for normal distribution. 
\end{proof}

\begin{corollary}
\label{cor:event}
$\prob\left[ \wedge_{\ell=0}^{k+1} {\cal E}_\ell\right] \geq \frac{1}{40}    k^{-10/\delta^2}.$
\end{corollary}

\begin{proof}
Applying union bound to~\Cref{fact:1} and~\Cref{fact:2}, we get 
\begin{equation}
\label{104}
\prob\left( \wedge _{\ell=0}^{k}{\cal E}_\ell \right)\geq 1-\sum_{\ell=0}^{k}\prob(\neg {\cal E}_\ell)
\geq 1/4 
\end{equation}

Since $u$ is Gaussian, $\lambda , z$ can be viewed as independent gaussian random variables. So we have that ${\cal E}_{k+1}$ is independent of ${\cal E}_\ell,\ell=0,1,2,\ldots , k$. Therefore,~\Cref{fact:3} and \eqref{104} imply that 
\begin{equation*}
\prob\left( \wedge_{\ell=0}^{k+1} {\cal E}_\ell\right)\geq  \frac{1}{40}    k^{-10/\delta^2}.
\end{equation*}
\end{proof}

For rest of the proof, we assume that the event $\wedge _{\ell=0}^{k+1}{\cal E}_\ell$ occurs. Now, for all $\ell \in [k]$, 
\begin{equation}\label{722}
u\cdot (\zeta_\ell-b) \stackrel{\eqref{eq:step2}}{\leq} z\cdot (\zeta_\ell-b)\leq
2\sqrt{\ln k} \dia(K),
\end{equation}
where the last inequality follows from the fact that the event $\cE_\ell$ has occured.
Also, since $z\perp (a-b)$:
\begin{equation}\label{723}
u\cdot (a-b)=\lambda |a-b|\geq 3\sqrt{\ln k} \dia(K),
\end{equation}
where the last inequality follows from the fact that the event $\cE_{k+1}$ has occured and $|a-b| \geq \delta \dia(K)$. 
The above two inequalities along with~\eqref{eq:step1} imply that 
$$u\cdot a - \max_{y\in K}u\cdot y \geq \lambda |a-b|/3.$$

Now, event $\cE_0$  implies that 
$$|u|\leq \lambda + |z|\leq \lambda+4\sqrt m, $$
and so, 
$$ u\cdot a - \max_{y\in K}u\cdot y  \geq \frac{\lambda \delta \dia(K) |u|}{3(\lambda + 4 \sqrt{m})},$$ 
The desired result now follows from~\Cref{cor:event} and the fact that $\lambda \cdot \delta \geq 3 \sqrt{\ln k}$ (\Cref{fact:3}). 
\end{proof}

\section{From $\opto_\varepsilon(K)$ oracles to the $\Hausdorf$ Problem}
\label{sec:haus}
% In this section, we show that we can get a result similar to that of~\Cref{upper-bound-1} while removing the well-separatedness condition. The well-separatedness assumption for a polytope $K$ is useful because for any vertex $v$ of $K$, there is a non-trivial probability of finding a direction $u$ along which the projection of $v$ is well-separated from that of other vertices of $K$. Hence, if we remove this assumption, we cannot hope to find a good approximation to each vertex of $K$ (for example, consider the case when $K$ closely approximates the unit circle on the two-dimensional plane). Instead we would like to approximate $K$ by another polytope $K'$, which may have many more vertices than $K$. We shall use the Hausdorff distance to measure the distance between two polytopes. 
We prove~\Cref{hausdorf-0} in this section. We begin by defining Hausdorff distance formally.

\begin{defn}
\label{defn:Haus}
    The {\em Hausdorff-distance}, $\Haus(K,K')$, between two polytopes $K$ and $K'$ is the infimum over all values $\alpha$ such that the following condition is satisfied: for every point $x \in K$, there is a point $y \in K'$ such that $|x-y| \leq \alpha$, and vice versa.
\end{defn}

%
%We first restate~\Cref{Hausdorf-problem} in more precise terms:
We now give the formal version of~\Cref{hausdorf-0}:
\begin{theorem}\label{upper-bound-3}
Suppose $\varepsilon,\delta$ are reals in $[0,1]$ with 
\begin{align}
\label{eq:deltacond}
    \delta > c \varepsilon \sqrt{d}, \delta > c/\sqrt{d}, 
\end{align}
where $c$ is a large enough constant. 
  Let $K$ be a $k$-vertex polytope with $\Delta(K)$ denoting the diameter of $K$. Suppose we are also given an  $ \opto_\varepsilon(K)$ oracle $\cal O$.  
    Let 
  $P$ be the set of answers from the oracle on $m:= k^{10+c \delta^{-2}}$ independent random queries. Then, $\Haus(K, \CH(P)) \leq \delta \cdot \Delta(K)$. 
\end{theorem}

\begin{proof}
  We describe the algorithm for obtaining the desired set $S$ (referred as {\bf Random Probes Algorithm } in~\Cref{algorithmic-rsh})  in~\Cref{algo:S}. The set $P$ is constructed as follows: we pick a set  of $m$ random unit vectors. For each such unit vector $u$, we add the corresponding point $x(u)$ returned by the  oracle $\cal O$ to the set $P$. 

  \begin{algorithm}[H]
  \caption{Algorithm for finding the set $P$ such that $\CH(P)$ approximates $K$.}
  \label{algo:S}
    {\bf Input:} An $\opto_\varepsilon(K)$ oracle  $\cal O $. \;
        Initialize a set $P$ to $\emptyset$ \;
       {\bf Repeat} $m$ times:  \;
            \quad Let $u$ be a random unit vector in $\Re^d$. \;
            \quad Call $\cal O$ on $u$ to get a vector $x(u)$. \;
            \quad Add $x(u)$ to $P$. \;
        {\bf Output} $P$. 
\end{algorithm}

One side of the desired result is easy to show:
\begin{claim}
    \label{cl:Haus1}
    For each $x \in \CH(P)$, there is a $y \in K$ such that $|x-y| \leq \delta \Delta(K)$. 
\end{claim}
\begin{proof}
    For a point $x(u) \in P$, we know by the definition of $\opto_\varepsilon(K)$ oracle that there is a point $y \in K$ such that 
    $|x(u) - y| \leq \varepsilon \Delta(K) \leq \delta \Delta(K)$, where $\varepsilon \leq \delta$ follows from~\eqref{eq:deltacond}. The desired result now follows from the convexity of $K$. 
\end{proof}

It remains to show that for any vertex $v$ of $K$, there is a point $y \in \CH(P)$ such that $|v-y| \leq \delta \Delta(K)$. Fix such a vertex $v$ of $K$ for rest of the discussion. Let the random unit vectors considered  in~\Cref{algo:S} (in the order they get generated) be $u^1, \ldots, u^m$. Let $P^j$ denote the subset $\{x(u^1), \ldots, x(u^j)\}$ of $P$. Define an event $\cE_j$ as follows:
$$ \dist(v, \CH(P^j)) \leq \delta \cdot \Delta(K) \quad \mbox{or} \quad \dist(v, \CH(P^{j+1})) \leq \left(1 - \frac{\delta^2}{c'} \right) \dist(v, \CH(P^j)),$$
where $c'$ is a large enough constant. Our main technical result is to show that conditioned on {\em any} choice of $u^1, \ldots, u^j$ the event $\cE_j$ happens with reasonably high probability (where the probability is over the choice of $u^{j+1}$)$:$
\begin{lemma}
    \label{lem:tech1}
    For any index $j \in [m-1]$, 
    $$ \Pr_{u^{j+1}} \left[\cE_j|u^1, \ldots, u^j \right] \geq \frac{1}{100}. $$
\end{lemma}
\begin{proof}
    Fix the vectors $u^1, \ldots, u^j$. If $\dist(v, \CH(P^j)) \leq \delta \cdot \Delta(K)$, then we are done. So assume this is not the case. Let $b$ be the closest point in $\CH(P^j)$ to $v$. Thus, 
    \begin{align}
        \label{eq:b}
        |v-b| \geq \delta \cdot \Delta(K).
    \end{align}Define $w$ as $$ w:= \frac{v-b}{|v-b|}.$$
 We can now express the vector $u^{j+1}$ as $\lambda w + z$, where $\langle z, w \rangle = 0$. We first show the following useful properties of these vectors. 

 \begin{claim}
     \label{cl:tech2}
     With probability at least $\frac{1}{100}$, the following three events happen:
     \begin{align}
         |z| & \leq 4 \sqrt{d} \label{eq:event1} \\
         \max_{y \in K} |z \cdot (v-y)| & \leq 2 \sqrt{\ln k} \Delta(K) \label{eq:event2} \\
         \lambda & \geq \frac{100}{\delta} \sqrt{\ln k} \label{eq:event3}
     \end{align}
 \end{claim}
    \begin{proof}
        The proofs of these three inequalities are identical to those of~\Cref{fact:1},~\Cref{fact:2} and~\Cref{fact:3} (in order to prove~\eqref{eq:event2}, it suffices to show it for points $y$ which are vertices of $K$). 
    \end{proof}

The following fact is also easy to show:
\begin{fact}
    \label{fact:easy}
$$ |v-x(u^{j+1})| \leq 2 \Delta(K).$$
\end{fact}
\begin{proof}
    By the definition of $\opto_\varepsilon(K)$ oracle, there is a point $p \in K$ such that $|p-x(u^{j+1})| \leq \delta \cdot \Delta(K) \leq \Delta(K)$. The desired result now follows by triangle inequality. 
\end{proof}

 Let  $\delta_1$ denote $\frac{\delta^2}{100}$. Let $b_1$ denote the vector
 $$ \delta_1 x(u^{j+1}) + (1-\delta_1)b. $$
 Since $b_1 \in \CH(P^{j+1})$, the desired result will follow if we prove the following:
 \begin{align}
     \label{eq:tech4}
     |v-b_1|^2 \leq \left(1- \frac{\delta^2}{100} \right) |v-b|^2. 
 \end{align}

 Now, 
 \begin{align}
     \notag
     |v-b_1|^2 & = \delta_1^2 |v-x(u^{j+1})|^2 + (1-\delta_1)^2 |v-b|^2 + 2\delta_1(1-\delta_1) (v-x(u^{j+1}))\cdot (v-b) \\
    \notag & \stackrel{\mbox{\small{\Cref{fact:easy}}}}{\leq} 4 \delta_1^2 \Delta(K)^2 + (1-2\delta_1) |v-b|^2 + \delta_1^2 \Delta(K)^2 + 2\delta_1(1-\delta_1) (v-x(u^{j+1}))\cdot (v-b) \\
    \notag &  \stackrel{\eqref{eq:b}}{\leq} \left( 1 - \frac{3\delta_1}{2} \right) |v-b|^2 +  + 2\delta_1(1-\delta_1) |v-b| \cdot (v-x(u^{j+1}))\cdot w \\
    \notag & = \left( 1 - \frac{3\delta_1}{2} \right) |v-b|^2 +  + 2\delta_1(1-\delta_1) \left( \underbrace{\frac{ |v-b|}{\lambda} \cdot (v-x(u^{j+1}))\cdot u^{j+1}}_{:=A}  \right. \\ \label{eq:AB}
  & \quad \quad  \left. \underbrace{-\frac{ |v-b| }{\lambda}\cdot (v-x(u^{j+1}))\cdot z} _{:=B} \right)
 \end{align}
 We now bound each of the terms $A$ and $B$ above. Now, 
 $$ A \leq \frac{|v-b|}{\lambda} \varepsilon \Delta(K) \stackrel{\eqref{eq:deltacond}}{\leq} \frac{|v-b| \delta \Delta(K)}{c \lambda} \stackrel{\eqref{eq:b}}{\leq} \frac{|v-b|^2}{c \lambda} \stackrel{\eqref{eq:event3}}{\leq} \frac{|v-b|^2 \delta}{c},$$
 where the first inequality follows from the definition of $\opto_\varepsilon(K)$ oracle. We now bound the quantity $B$. Let $y$ be the point in $K$ closest to $x(u^{j+1})$. We know that $|y-x(u^{j+1})| \leq \varepsilon \Delta(K)$. Therefore, 
 \begin{align*}
     B & \leq \frac{|v-b|}{\lambda} \left( |(v-y) \cdot z| + |z| \varepsilon \Delta(K) \right) \\
     & \stackrel{\eqref{eq:event1}, \eqref{eq:event2}}{\leq} \frac{|v-b|}{\lambda} \left( 2 \sqrt{\ln k} \Delta(K) + 4 \varepsilon \sqrt{d} \Delta(K) \right) \\
     & \stackrel{\eqref{eq:deltacond}, \eqref{eq:event3}}{\leq} \frac{\delta |v-b| \Delta(K) }{10} \stackrel{\eqref{eq:b}}{\leq} \frac{|v-b|^2}{10}. 
 \end{align*} 
 Substituting the above bound on $A$ and $B$ in~\eqref{eq:AB} yields the desired result. 
\end{proof}
We are now almost done. As the following result shows, it suffices to argue that enough number of events $\cE_j$ happen:
\begin{claim}
\label{cl:close}
    If at least $\frac{c'}{\delta^2} \ln(2/\varepsilon)$ of the events $\cE_j, j \in [m-1]$ happen, then $\dist(v, \CH(P)) \leq \delta \Delta(K)$. 
\end{claim}
\begin{proof}
Assume, for the sake of contradiction, that $\dist(v,\CH(P)) > \varepsilon \Delta(K)$. 
Assume that events $\cE_{j_1}, \ldots, \cE_{j_h}$ happen, where $h :=\frac{c'}{\delta^2} \ln(2/\varepsilon). $ Now, for any index $i \in [h-1]$, the definition of $\cE_{j_{i+1}}$ implies that 
$$\dist(v, \CH(P^{j_{i+1}})) \leq \left( 1 - \frac{\delta^2}{c'} \right) \dist(v, \CH(P^{j_{i+1}-1}) \leq  \dist(v, \CH(P^{j_{i}})).$$
Therefore, 
$$\dist(v, \CH(P^{j_h})) \leq \left( 1 - \frac{\delta^2}{c'} \right)^h \dist(v, P^1) 
\stackrel{\mbox{\small{\Cref{fact:easy}}}}{\leq} \left( 1 - \frac{\delta^2}{c'} \right)^h 2 \Delta(K) \leq \delta \Delta(K). $$
\end{proof}

It remains to show that with high probability at least $h :=\frac{c'}{\delta^2} \ln(2/\varepsilon) $
of the events happen. In order to prove this, we divide the sequence $[m]$ into $[h]$ subsequences, each of length $m/h$. Call these subsequences $C_1, \ldots, C_h$. It follows from~\Cref{lem:tech1} that for any $i \in [h]$, 
$$ \Pr \left[ \wedge_{j \in C_i} \neg \cE_j \right] \leq 0.99^{m/h} \leq \frac{1}{h^2}.$$
A simple union bound now shows that with probability at least $1-1/h$, at least one event $\cE_j$ happens during each of the subsequences $C_1, \ldots, C_h$.~\Cref{cl:close} now proves the theorem. 
% First, we relate the parameters here to the parameters we had in our write-up. Under ``Assumptions on model and generation process'', (A1): Mean Separation, we had (with Diameter $\leq 1$ and $\Gamma $ replaced by $\delta$, the inequality we had connecting parameters (it was (18)) in the Dec 16 old write-up):
% \begin{equation}\label{791}
% (i) \delta \in \Omega(\sqrt k\sigma_0/\sqrt{w_0})\quad ;\quad (ii) \delta \in \Omega^*(1).\end{equation}
% Now, the error $\varepsilon$ is bounded above by $\sigma_0/\sqrt{w_0}$ in our old write-up since for a set $S$ of $w_0n$ elements, we have $|A_{\cdot,S}-P_{\cdot,S}|\leq \sigma_0/\sqrt{w_0}$. We will see here (using Thm 3.1 above in the sVD projection) that $\varepsilon\in \Omega(1/\sqrt k)$ suffices. To get this we need 
% $$\sigma_0\in O(\sqrt{w_0}/\sqrt k)$$ which is exactly inequality (i) in (\ref{791}) since $\delta \in\Omega(1)$.

% We proved already that if hat denotes projection to SVD subspace, (Lemma 4.5)
% $$|\widehat M_{\cdot,\ell}-M_{\cdot,\ell}|\leq \sigma_0/\sqrt {w_0}.$$
% Thus, $\delta $ is changed by at most a factor of 2 when we project.

% Note that by subset smoothing, we have an optimization oracle in ${\bf R}^d$ with error $\varepsilon\leq \sigma_0/\sqrt{w_0}$ which is at most $O(1/\sqrt k)$.
% Now for any $u$ in the SVD subspace, the same oracle automatically serves as an opt oracle over the projection with at most the same error.

\end{proof}

\section{From $\opto_\varepsilon(K)$ oracles to $\ListLearn$}
\label{sec:listlearn}

We first define the notion of well-separatedness. 
\begin{defn}
    We say that a polytope $K$ with vertex set $V$ is {\em $\delta$-well-separated} if for every vertex $v \in V$, we have
    $$ \dst(v,\CH(V \setminus \{v\})) \geq \delta \cdot \Delta(K),$$
    where $\Delta(K)$ denotes the diameter of $K$.
\end{defn}

We first show the lower bound result~\Cref{lower-bound-0}, which is restated here. 

\lowerbound*

\begin{proof}
    The proof is by producing an ``adversarial oracle''. The parameter $\varepsilon$ is set to $8 \ln d/\sqrt{d}$. Our $K$ will a ``needle'' of the form $\{\lambda u:\lambda\in [-1,1]\}$, where, $u$ is a unit length vector. The oracle's answer for each query will  always be the 0 vector. Let  $v_1,v_2,\ldots, v_q,$ where $q=d^c$ for a constant $c$, be the vectors on which the oracle is queried. The answer 0 is clearly a valid answer for any $\opto_\varepsilon(K)$ oracle,  provided  $u$ is a unit vector with $|u\cdot v_i|\leq 4\ln d/\sqrt d$. For a single $v_i$, the probability that a random $u$ satisfies $|u\cdot v_i|>4\ln d/\sqrt d$ is at most $e^{-d\varepsilon^2}$. So by union bound, we see that there is a vector $u_1$ satisfying 
    $$|u_1\cdot v_i|\leq 4\ln d/\sqrt d\forall i\in [q].$$
    Further, by a similar argument, there is a vector $u_2$satisfying
    $$|u_2\cdot v_i|\leq 4\ln d/\sqrt d, i\in [q]\; ;\; |u_1-u_2|,|u_1+u_2|\geq .1.$$ It follows that the two possible $K$, where 
    $$\{ \lambda u_1:\lambda\in [-1,1]\}\; ,\; \{ \lambda u_2:\lambda\in [-1,1]\}$$
are both consistent with the answer 0 for all $v_i,i\in [q]$ and in addition, we have that no point $w\in {\bf R}^d$ is within distance $\varepsilon$ of two vertices, one from each of these 2 needles. So no answer is valid for both needles and the adversarial oracle can choose one of the needles to render the algorithm's answer incorrect.
\end{proof}

We now formally prove~\Cref{list-0}, which gives an algorithm for the $\ListLearn$ problem.

\begin{theorem}\label{upper-bound-1}
Suppose $\varepsilon,\delta$ are reals in $[0,1]$ with $\delta^2\geq c\varepsilon\sqrt d, \delta^3 \geq c \, \varepsilon,$ where $c$ is a large enough constant.  Let $K$ be a $\delta$-well-separated $k$-vertex polytope. Suppose we are also given  $ \opto_\varepsilon(K)$ oracle $\cal O$. 
Let $W$ be the set of answers of the oracle to $m=poly(d)\cdot k^{\Omega(1/\delta^2)}$ independent random queries. Then for each vertex $v$ of $K$, there is a point $v' \in W$ such that $|v-v'| \leq O(\delta^2 \Delta(K)/c).$
\end{theorem}
\begin{proof}
 The algorithm chooses a set $U$ of $k^{\Omega(1/\delta^2)}$ unit length i.i.d. Gaussian vectors. For each $u\in U$, it calls the oracle $\cal O$ to find a vector $x(u)$. Let $W$ denote the set $\{x(u): u \in U\}$. We invoke~\Cref{cor:soft} with the parameters: 
$$ \delta' := \delta/4, \, \varepsilon' := \frac{32 \delta^2}{c}$$
on the set $W$. The algorithm in that result outputs a set $Q$. We output that $Q$ as the approximation to the set of vertices of $K$. We now prove that the set $Q$ has the desired properties. We first show that for every vertex of $K$, there is a direction $u$ in $U$ along which the projection of this vertex is higher than the projection of the remaining vertices by a large enough margin. 
Let $\Mdot{1}, \ldots, \Mdot{k}$ be the vertices of $K$. 
\begin{claim}
    \label{cl:ul}
    With high probability, the following event happens: for each $\ell \in [k]$, there is a vector $u^{(\ell)} \in U$ such that for all $\ell' \in [k], \ell' \neq \ell,$ we have
    \begin{align}
        \label{eq:100}
        u^{(\ell)}\cdot M_{\cdot, \ell}> u^{(\ell)}\cdot M_{\cdot, \ell'}+\frac{c \, \varepsilon \cdot \Delta(K)}{8 \delta^2}
    \end{align}
\end{claim}

\begin{proof}
    Fix a vertex $\Mdot{\ell}$ of $K$.  Let $K'$ be the convex hull of $\{\Mdot{1}, \ldots, \Mdot{k} \} \setminus \{\Mdot{\ell}\}.$ We invoke~\Cref{thm:sephp} on the polytope $K'$ and the point $a := \Mdot{\ell}$. The definition of $\delta$-well-separated implies that~\eqref{eq:condhp} is satisfied with $\rho = \delta$. Using $V = \Re^d$ in the statement of~\Cref{thm:sephp}, we see that $$\Pr_{u \in U} \left[ u \cdot \Mdot{\ell} - \max_{\ell' \in [k], \ell' \neq \ell} u \cdot \Mdot{\ell'} \geq \frac{\delta \sqrt{\log k} \Delta(K)}{\sqrt{\log k} + 4 \delta \sqrt{d}} \right] \geq \frac{1}{40} k^{-10/\delta^2}.$$
    Since 
    $$\frac{\delta \sqrt{\log k}}{\sqrt{\log k} + 4 \delta \sqrt{d}}  \geq \min \left( \frac{\delta}{2}, \frac{1}{8 \sqrt{d}} \right) \geq \frac{\varepsilon c}{8 \delta^2}, $$
    the desired result follows from the fact that $|U| \gg \frac{1}{40} k^{-10/\delta^2}.$
    
\end{proof}

For rest of the proof, assume that the statement in~\Cref{cl:ul} holds true, i.e., there are directions $u^{(1)}, \ldots, u^{(k)} \in U$ satisfying~\eqref{eq:100}. 
We now show that for every vertex of $\Mdot{\ell}$ of $K$, the corresponding point $x(u^{\ell})$ is close to  $\Mdot{\ell}$. 
\begin{claim}
    \label{cl:xl}
    For every $\ell \in [k]$, 
    $$|x(u^{(\ell)})-M_{\cdot,\ell}|\leq 17 \delta^2 \Delta(K)/c.$$
\end{claim}
\begin{proof}
    By the definition of $\cal O$, we know that $x(u^{(\ell)})$ can be written as $y(u^{\ell})+z(\uell)$, where $y(\uell) \in K$ and $|z(\uell)| \leq \varepsilon \Delta(K)$. Thus, there is a convex combination $\lambda_{\ell'}, \ell' \in [k],$ of the vertices $\Mdot{\ell'}$ of $K$ such that 
    $$ x(\uell) = \sum_{\ell' \in [k]} \lambda_{\ell'} \Mdot{\ell'} + z(\uell).$$

By the definition of $\cal O$, $x(u^{(\ell)})\cdot u^{(\ell)}\geq M_{\cdot, \ell}\cdot u^{(\ell)}-\varepsilon\Delta (K)$ and $|z(u^{(\ell)})\cdot u^{(\ell)}|\leq\varepsilon\Delta(K)$. So, we get
$$M_{\cdot,\ell}\cdot u^{(\ell)}-\varepsilon\Delta(K)\leq \sum_{\ell' \in [k]}\lambda_{\ell'} M_{\cdot,\ell'}\cdot u^{(\ell')}+\varepsilon\Delta(K),$$
which implies (after subtracting $\lambda_\ell M_{\cdot,\ell} \cdot u^{(\ell)}$ from both sides):
$$(1-\lambda_\ell) M_{\cdot,\ell}u^{(\ell)}\leq \sum_{\ell'\not= \ell} \lambda_{\ell'} M_{\cdot,\ell'}\cdot u^{(\ell')}+2\varepsilon\Delta(K)$$
which, using~\Cref{cl:ul}, yields:
$$(1-\lambda_\ell) M_{\cdot,\ell}u^{(\ell)}\leq (1-\lambda_\ell) M_{\cdot,\ell}u^{(\ell)}-(1-\lambda_\ell)\frac{c\varepsilon\Delta(K)}{8\delta^2}+2\varepsilon\Delta(K).$$
%which leads to:

It follows from the above inequality that 
\begin{align}
\label{eq:105}
    1- \lambda_\ell \leq \frac{16 \delta^2}{c}.
\end{align}
    Therefore, 
\begin{align}
    |x(\uell) - \Mdot{\ell}| = \left| \sum_{\ell' \neq \ell} \lambda_{\ell'} (\Mdot{\ell} - \Mdot{\ell'})  
    \right| + \varepsilon  \Delta(K) \leq \sum_{\ell' \neq \ell} \lambda_{\ell'} \Delta(K) + \varepsilon \Delta(K) \stackrel{\eqref{eq:105}}{\leq} \frac{17 \delta^2 \Delta(K)}{c}.
    \notag
\end{align}
%where the second inequality follows from the fact that $\Delta(K) \leq 1$. 
\end{proof}

This completes proof of the Theorem. 
\end{proof}

The statement of~\Cref{upper-bound-1} requires two different bounds relating $\delta$ to $\varepsilon$. In some applications, it may be difficult to ensure both of these conditions. We consider the setting when $d \gg k$, where the following variation of~\Cref{upper-bound-1} is better suited. 

\begin{theorem}\label{upper-bound-2}
Suppose $\varepsilon,\delta$ are reals in $[0,1]$ with $\delta^2\geq c\varepsilon\sqrt d, \delta \geq \frac{ \sqrt{\log k}}{\sqrt{cd}},$ where $c$ is  large enough constant.  Let $K$ be a $\delta$-well-separated $k$-vertex polytope. Suppose we are also given an  $\opto_\varepsilon(K)$ oracle $\cal O$.
Let $W$ be the set of answers of the oracle to $m=poly(d)\cdot k^{O(1/\delta^2)}$ independent queries.  Then for each vertex $v$ of $K$, there is a point $v' \in W$ such that $|v-v'| \leq O(\delta^2 \Delta(K)/\sqrt{c}).$
\end{theorem}
\begin{proof}
    The proof is identical to that of~\Cref{upper-bound-1}, except that in the proof of~\Cref{cl:ul}, we now have the following modified argument. Since $\delta \sqrt{cd} \geq \sqrt{\log k}, $ we get 
     $$\frac{\delta \sqrt{\log k}}{\sqrt{\log k} + 4 \delta \sqrt{d}}  \geq \frac{1}{8 \sqrt{cd}} \geq \frac{\varepsilon \sqrt{c}}{8 \delta^2}. $$
     Rest of the proof of~\Cref{cl:ul} follows without any further changes (where we replace $c$ by $\sqrt{c}$). 
\end{proof}

\section{From $\ListLearn$ to the  $\kolp$ Problem}
\label{sec:kolp}

In this section, we show that for well-separated polytopes, a solution for the $\ListLearn$ problem can be used to solve the $\kolp$ problem as well. This algorithm uses the notion of soft convex hulls. 
We first describe the algorithm for constructing soft convex hulls, and then use it to solve the $\kolp$ problem. 

\subsection{Soft Convex Hulls}
\label{sec:soft}
Let $W$ be a finite set of points in $\Re^d$, and
$T$ be the vertices of $\ch(W)$. The subset $T$ of $W$ is the unique subset of $W$ with the following properties: 
\begin{itemize}
\item[(P1)] $W\subseteq \ch(T)$
\item[(P2)] $\forall w\in W$, if $w\notin \ch(W\setminus \{w\})$, then $w \in T$. 
\end{itemize}

We now define a natural notion of {\em soft} convex hull. 
\begin{defn}\label{soft-CH}
For an $\varepsilon \geq 0$, and $S\subseteq W$, define the $\varepsilon$-convex hull of $S$, $\ech(S)$, as $\ch(S) + \varepsilon \dia(W) B$, where $B$ is the unit ball of the Eucliean norm. 
\end{defn}

 The intuition behind the above definition is that $\ch(W)$ can have a very large set of vertices, but there may be a small set of points whose soft convex hull contains $W$. This is defined more formally as follows: 

\begin{defn}\label{envelope}
We call a subset $T\subseteq W$ an $\varepsilon$-envelope of $W$, written
$\eenv(W)$, if 
$W\subseteq \ech(T).$  
\end{defn} 

\noindent
{\bf Remarks:} The following observations about the set $\eenv(W)$ are easy to see:
\begin{itemize}
\item[(a)] There are several distinct sets $T$ which could qualify as $\eenv(W)$. For example, let $W$ consist of the following set of points in $\Re^2$:  a set of points $W_1$ close to $(0,0)$ and a set of points $W_2$ close to $(1,0)$. Let $T$ be a pair of points $\{x,y\}$ with $x \in W_1, y \in W_2$. Then it is easy to check that $T$ is an $\eenv(W)$. 
\item[(b)] Let $T$ be $\eenv(W)$. Unlike property~(P2) above, it is not necessary that if $w \in W$ is such that $w \notin \ech(W \setminus \{w\})$, then $w \in T$. 
 \end{itemize}

Since the set $\eenv(W)$ is not uniquely determined, we will impose one more condition on it to make it unique (if it exists) and polynomial time computable. This condition requires
the points of $T$ to be ``far apart'' from each other. More precisely:

\begin{defn}\label{strict-envelope}
For $\varepsilon,\delta \in [0,1]$, a set $T$ is called a $\edenv(W)$ if it is an $\eenv(W)$ and  
\begin{equation}\label{792}
\forall w\in T, \dist(w,\ch(T\setminus \{w\}))> \delta \dia(W)
\end{equation}
\end{defn}

\begin{fact}
\label{f:check}
Given a subset $T$ of $W$, we can check in polynomial time whether $T$ is an $\edenv(W)$. 
\end{fact}

\begin{proof}
Fix a subset $T$. We can verify in polynomial time whether $T$ is an $\eenv(W)$. Indeed, for each point $w \in W$, we need to check if $w \in \ech(T)$. This can be expressed as a convex program feasibility problem, where the variables are $\lambda_t$ for each $t \in T$: 
$$ \left|\sum_{t \in T} \lambda_t \cdot t - w \right| \leq \varepsilon \dia(W), \quad \sum_t \lambda_t = 1, \quad \lambda_t \geq 0 \  \  \forall t \in T.$$ 

Similarly, we can check~\eqref{792} using convex programming. For each $w \in T$, we can find $\dst(w, \ch(T \setminus  \{w\})$ as follows, where the variables are $\lambda_x, x \in T \setminus \{w\}$: 
$$ \min. \ \left| \sum_{x \in T \setminus \{w\}} \lambda_x \cdot x - w \right|, \quad \sum_x \lambda_x = 1, \quad \lambda_x \geq 0 \  \  \forall x \in T\setminus \{w\}. $$
\end{proof}

For rest of the section, we address the following question: given the set $W$, and parameters $\varepsilon, \delta$, is there a  $\edenv(W)$, and if so, can we find in polynomial time an approximation to this set? We informally argue that several ``natural'' greedy strategies do not work. Consider, for example, the following algorithms for identifying $\edenv(W)$: 
\begin{itemize}
    \item Identify $T$ as the set of points $w\in W$ for which $w$ is not close to $\ch(W\setminus \{w\})$: We define a set of points in $\Re^2$. We have {\em rings} of points, one around $ (0,0)$ and the other around $(0,1)$, each point is close to the convex hull of the others and so this algorithm will   dismiss all the points as not belonging to the desired set $T$.
    % \item \alert{This example is not clear to me, i.e., I am not sure all the conditions on the 4 points can be satisfied. - rk   }Start by adding an extreme point (along a randomly chosen direction) to $T$. For the next $k-1$ steps, iteratively find a point $w$ for which $\dist(w,T)$ is maximized and add it to $T$: We consider a set of points in $\Re^2$ as follows: let $v_1, v_2, v_3$ be three affinely independent points with $|v_1 - v_2| > |v_1 - v_3|/\delta, |v_2 - v_3|/\delta$ and $\dist(v_3, \ch(\{v_1, v_2\})) > \delta |v_1 - v_2|$. Now let $v_4$ be the mid-point of the line segment joining $(v_1, v_2)$. Let $W=\{v_1, v_2, v_3, v_4\}$. Say the first point selected by the greedy algorithm is $v_1$. Then the next two points selected by it will be $v_2, v_4$ respectively. But the ``correct'' $T$ is $\{v_1, v_2, v_3\}$. 
    \item Start by adding an extreme point (along an arbitrarily chosen direction) to $T$. For the next $k-1$ steps, iteratively find a point $w$ for which $\dist(w,\ch(T))$ is maximized and add it to $T$: a simple example as above shows that this idea may not work as well. Consider $v_1, v_2, v_3, v_4$ lying on corners of a square, and let $v_5$ be the mid-point of $v_3, v_4$. We start by adding $v_1$ and then $v_2$ to $T$. But in the next step, we could add $v_5$, and then both $v_3, v_4$. But the ``correct'' $T$ would have been $\{v_1, v_2, v_3, v_4\}$. 
\end{itemize}

If $\varepsilon=0$, the above question is easy to answer in polynomial time. The answer is yes iff the set $T$ of vertices of $\ch(W)$ satisfies (\ref{792}). Also, if $\delta=1$, $T$ has to be a singleton to satisfy (\ref{792}). 

In rest of this section, we consider the following problenm: For what pairs of values of $\varepsilon,\delta,$ can we prove that there is {\em essentially} at most one $\edenv(W)$, and if so, can we determine this set efficiently ? 
We do not know the exact answer to this, but our main result here (which suffices for the applications) is 
(verbally stated) an affirmative answer to the question if the following condition is satisfied:
$$\delta\in \Omega(\sqrt\varepsilon).$$
This will follow as a corollary of our main result:

\begin{theorem}
\label{thm:soft}
Let
$\delta,\varepsilon,\varepsilon_3$ be reals in $(0,1/8)$ satisfying
\begin{equation}
\label{822}\delta> \max \left(\frac{2\varepsilon}{\varepsilon_3-\varepsilon}, 4 \varepsilon_3 \right)
\end{equation}
Let $W$ be a finite set of points in $\Re^d$. 
We can 
determine in polynomial time whether  exists a set $T$ in $\edenv(W)$, and if so,  we can efficiently 
find a subset $Q$ of $W$ such that
\begin{align}
|Q|=|T|\\
\forall w\in T, \exists x\in Q: |w-x|\leq 2\varepsilon_3 \dia(W)
\end{align}
\end{theorem}
\begin{corollary}\label{cor:soft}
Let
$\delta,\varepsilon$ be reals in $(0,1/8)$ satisfying $\delta>16\sqrt\varepsilon$
Let $W$ be a finite set of points in $\Re^d$. 
We can 
determine in polynomial time whether  exists a set $T$ forming a $\edenv(W)$, and if so,  we can efficiently 
find a subset $Q$ of $W$ such that
\begin{align}
|Q|=|T|\\
\forall w\in T, \exists x\in Q: |w-x|\leq 8\sqrt \varepsilon \dia(W)
\end{align} 
\end{corollary}
\begin{proof}
The Corollary follows from Theorem (\ref{thm:soft}) by taking $\varepsilon_3=4\sqrt\varepsilon$ 
\end{proof}
\begin{proof} (of \Cref{thm:soft})
The procedure is described  in~\Cref{algo:main}. We first compute a subset $Q''$ of $W$ consisting of points $w \in W$ which do not lie in the soft convex hull of the points in $W$ which are ``far'' from $w$ -- this can be doen in polynomial time by using arguments similar to those in the proof of~\Cref{f:check}. Then $Q$ is defined as a maxumal subset of points in $Q''$ with the property that the pair-wise distance between the points in it large. Finally, we check whether $Q$ is an $\edenv(W)$, which can be done efficiently using~\Cref{f:check}.

\begin{algorithm}[H]
  \caption{Procedure to identify $\edenv(W)$}
  \label{algo:main}
    Compute $\dia(W)$. \;
    Let $Q' := \{w \in W: w \in \ech \left(\{x \in W: |w-x| \geq \varepsilon_3 \dia(W) \}\right). $ \;
    Let $Q'' := W \setminus Q'$. \;
    Define $Q :=$ maximal subset of $Q''$ such that for every distinct $x,y \in Q$, $|x-y| > 2\varepsilon_3 \dia(W). $ \;
    \If{$Q$ is $\edenv(W)$} {
       Output  the set $Q$. \; }
    \Else {
     Output No. }
\end{algorithm}

Now we analyze this algorithm. If there is no $\edenv(W)$, then the algorithm will clearly say ``No'' (because the set $Q$ will not be $\edenv(W)$). So assume that there is a set $T$ which is $\edenv(W)$. 

\begin{claim}
\label{cl:ed}
$T \subseteq Q''$. 
\end{claim}
\begin{proof}

Suppose for the sake of contradiction that there is a point $w \in T \setminus Q''$.  For sake of brevity, let $W_x$ denote the subset of points in $W$ satisfying $|w-x|  \geq \varepsilon_3 \dia(W)$. The fact that $w \notin Q''$ implies that $w \in \ch(W_x)$, i.e., 
$$ w = \sum_{x \in W_x} \lambda_x \cdot x + e_0, $$
where $\lambda_x$ form a convex combination and $|e_0| \leq \varepsilon \dia(W).$ Since $W \subseteq \eenv(T)$, for each $x \in W_x$, there are a point $x' \in \ch(T)$ such that $e_x := x-x'$ has length at most $\varepsilon \dia(W)$. Since $w \in T$, we can write $x'$ as 
$$ x' = \mu_x w + (1-\mu_x) y_x, \quad y_x \in \ch(T \setminus \{w\}), $$
where $\mu_x \in [0,1]$. 
Therefore, 
$$ |w-x| \leq |w-x'| + |e_x| = (1-\mu_x) |w-y_x| + |e_x| \leq (1-\mu_x) \dia(W) + \varepsilon \dia(W). $$
But we know that $|w-x| \geq \varepsilon_3 \dia(W)$. Therefore, we get
\begin{align}
\label{eq:111}
(1- \mu_x) \geq \varepsilon - \varepsilon_3. 
\end{align}
Now, 
\begin{align*}
w & = \sum_{x \in W_x} \lambda_x \cdot x + e_0 = \sum_{x \in W_x} \lambda_x \cdot x' + \sum_x \lambda_x e_x + e_0 \\
& = \sum_{x \in W_x} \lambda_x \mu_x \cdot w + \sum_{x \in W_x} \lambda_x (1 - \mu_x) y_x + e,  
\end{align*}
where $e = \sum_x \lambda_x e_x + e_0$ has length at most $2\varepsilon \dia(W)$. 
Let $\theta_x$ denote $\lambda_x(1-\mu_x)$. Observe that $\theta:= \sum_x \theta_x = 1 - \sum_x \lambda_x \mu_x$. Therefore, we can rewrite the above as 
$$ w - \frac{1}{\theta} \sum_{x \in W_x} \theta_x y_x = \frac{e}{\theta}. $$
Since $\frac{1}{\theta} \sum_{x \in W_x} \theta_x y_x \in \ch(T \setminus \{w\})$, we see that 
$$ \dist(w, \ch(T \setminus \{w\}) \leq \frac{|e|}{\theta} \leq 
\frac{2\varepsilon \dia(W)}{\sum_{x \in W_x} \lambda_x (1-\mu_x)} 
\stackrel{\eqref{eq:111}}{\leq} \frac{2\varepsilon}{\varepsilon-\varepsilon_3} \dia(W). $$
Using~\eqref{792} and~\eqref{822}, we get a contradiction. Therefore, $w \in Q''$. 
\end{proof}

Since $T \subseteq Q''$, for every $w \in T$, there exists a point $x_w \in Q$ such that $|w-x_w| \leq 2 \varepsilon_3 \dia(W). $ It is easy to see that if $w, w'$ are two distinct points in $T$, then $x_w \neq x_{w'}$. Indeed, 
$$ |x_w - x_{w'}| \geq |w-w'| - |w-x_w| - |w'-x_{w'}| \geq \delta \dia(W) - 4 \varepsilon_3 \dia(W) \stackrel{\eqref{822}}{>} 0. $$
Finally, we prove that every element in $Q$ is of the form $x_w$ for some $w \in T$. 

\begin{claim}
\label{cl:final}
Every $x \in Q$ is of the form $x_w$ for some $w \in T$. 
\end{claim}
\begin{proof}
Consider a point $x \in Q$. This implies that $x \in Q''$. We claim that there is a point $w \in T$ such that $|x-w| \leq \varepsilon_3 \dia(W)$. Suppose not. Then the set of points $w$ such that $|x-w| \geq \varepsilon_3 \dia(W)$ includes $T$. But we know that $x \in \ech(T)$, and so $x \in Q'$, a contradiction. 

Therefore $$|x - w_x| \leq |x- w| + |w-w_x| \leq  2\varepsilon \dia(W). $$
This implies that $x=w_x$ (since any two distinct points in $Q$ have distance greater than $2\varepsilon \dia(W)$. 
\end{proof}

Thus we have shown that $|Q|=|T|$ and for every $w \in T$, there is a unique element $x_w \in Q$ with $|w-x_w| \leq 2 \varepsilon_3 \dia(W)$.
\end{proof}

\subsection{Algorithm for $\kolp$}
We now show how soft convex hulls can be used to generate a solution for the $\kolp$ problem. The following result, which formalizes~\Cref{kolp-0}, uses the same setting as that in~\Cref{upper-bound-1}:

\begin{theorem}\label{upper-bound-kolp}
Suppose $\varepsilon,\delta$ are reals in $[0,1]$ with $\delta^2\geq c\varepsilon\sqrt d, \delta^3 \geq c \, \varepsilon,$ where $c$ is a large enough constant.  Let $K$ be a $\delta$-well-separated $k$-vertex polytope. Suppose we are also given an  $ \opto_\varepsilon(K)$ oracle $\cal O$. 
Let $W$ be the set of answers of the oracle $\cal O$ to $m=poly(d)\cdot k^{\Omega(1/\delta^2)}$ independent random queries. We can find $Q\subseteq W, |Q|=k$ in 
randomized $poly(d) \cdot k^{\Omega(1/\delta^2)}$-time which satisfies the following condition w.h.p.:  for every vertex $v$ of $K$, there is a point $v'$ in $Q$ with $|v-v'| \leq \delta \Delta(K)/10$. 
\end{theorem}
\begin{proof}
    The proof of~\Cref{upper-bound-1} shows that  for every vertex  $\Mdot{\ell}$ of $K$, there is point $x(u^{\ell}) \in W$ such that 
    \begin{align}
        \label{eq:distMl}
          |x(u^{(\ell)})-M_{\cdot,\ell}|\leq 17 \delta^2 \Delta(K)/c.
    \end{align}
    Let $T$ denote the set of points $\{x(\uell): \ell \in [k]\}$. Our first claim is that the points of $T$ are also well-separated. 

\begin{claim}
\label{cl:wellstx}
For any $\ell \in [k]$, 
$$\dst(x(\uell), \CH(T \setminus \{x(\uell)\})) \geq \delta \Delta(K)/2.$$ Further, the diameter of $\CH(T)$ is at most $2 \Delta(K). $
\end{claim}
\begin{proof}
    Fix an index $\ell \in [k]$ and a point $y \in \CH(T \setminus \{x(\uell)\})).$ We can express $y$ as a convex  combination of points in $T \setminus \{x(\uell)\},$ i.e., 
    $$ y = \sum_{\ell' \neq \ell} \lambda_{\ell'} \cdot x(u^{(\ell')}), \quad \mbox{where \,  } \sum_{\ell' \neq \ell} \lambda_{\ell'} = 1.$$ Now
    \begin{align*}
        |x(\uell) - y| & \geq \left| \Mdot{\ell} - \sum_{\ell' \neq \ell} \lambda_{\ell'} \cdot \Mdot{\ell'} \right| - |x(\uell) - \Mdot{\ell}| - \sum_{\ell' \neq \ell} \lambda_{\ell'} \cdot |x(u^{(\ell')})-\Mdot{\ell'}| \\
        & \geq \left(\delta - \frac{34 \delta^2}{c} \right) \Delta(K) \geq \delta \Delta(K)/2,
    \end{align*}  
    where the second last inequality follows from~\eqref{eq:distMl} and the fact that $K$ is $\delta$-well-separated. 
    Since $\dst(x(\uell), K) \leq \varepsilon \Delta(K)$, it follows that $\Delta(\CH(T)) \leq 2 \Delta(K).$
    %A similar argument shows that for any two distinct $\ell, \ell' \in [k]$, 
    %$$ |x(\uell) - x(u^{(\ell')})| \leq (1+ \delta)\Delta(K) %\leq 2 \Delta(K).$$
\end{proof}

Recall that $W$ denotes the set $\{x(u): u \in U\}$. We now show that $\CH(T)$ closely approximates the set $\CH(W)$. 
\begin{claim}
    \label{cl:T}
    $W \subseteq \echp(T), $ where $\varepsilon'= \frac{32 \delta^2}{c}. $
\end{claim}

\begin{proof}
    Fix a point $x(u) \in W$. We know that $x(u)$ can be written as 
    $$ x(u) = y(u) + z(u), \quad y(u) \in K, |z(u)| \leq \varepsilon.$$
    Let $y(u) = \sum_{\ell \in [k]} \lambda_{\ell} \cdot \Mdot{\ell}, $ where the coefficients $\lambda_\ell$ form a convex combination. Then 
    $$ \left|x(u) - \sum_{\ell \in [k]} \lambda_\ell \cdot x(\uell) \right| \leq \sum_{\ell \in [k]} \lambda_\ell \cdot |x(\uell) - \Mdot{\ell}| + |z(u)| \leq \frac{17 \delta^2 \Delta(K)}{c} + \varepsilon \Delta(K) \leq \frac{32 \delta^2 \Delta(K)}{c},  $$
    where the second last inequality follows from~\eqref{eq:distMl} and the last inequality by the assumption in. This proves the desired result. 
\end{proof}

~\Cref{cl:wellstx} and~\Cref{cl:T} imply that $T$ is $\edenvp(W)$ with $\delta' = \delta/4, \varepsilon'= \frac{32 \delta^2}{c}$. We can now apply~\Cref{cor:soft} to get approximations to $x(\uell)$ within distance $17 \sqrt{\varepsilon'} \Delta(K)$.~\Cref{cl:T} now implies that we can get approximations to $\Mdot{\ell}$ within distance 
$$ 17 \sqrt{\varepsilon'} \Delta(K) + \frac{16 \delta^2 \Delta(K)}{c} \leq \frac{\delta \Delta(K)}{10}.$$
This proves the desired result.
\end{proof}

The following result is the analogue of~\Cref{upper-bound-2}, and gives an algorithm for the $\kolp$ problem under slightly different conditions on the parameters $\delta$ and $\varepsilon$. 

\begin{theorem}\label{upper-bound-2kolp}
Suppose $\varepsilon,\delta$ are reals in $[0,1]$ with $\delta^2\geq c\varepsilon\sqrt d, \delta \geq \frac{ \sqrt{\log k}}{\sqrt{cd}},$ where $c$ is  large enough constant.  Let $K$ be a $\delta$-well-separated $k$-vertex polytope. Suppose we are also given an $\opto_\varepsilon(K)$ oracle $\cal O$.
Let $W$ be the set of answers of the oracle $\cal O$ to $m=poly(d)\cdot k^{O(1/\delta^2)}$ independent queries.  In randomized $poly(d) \cdot k^{O(1/\delta^2)}$-time we can find $Q\subset W$ of $k$ points such that the following condition is satisfied w.h.p.:  for every vertex $v$ of $K$, there is a point $v'$ in $Q$ with $|v-v'| \leq \delta \Delta(K)/10$. 
\end{theorem}

\section{$\kolp$ algorithm for Latent Polytopes using Singular Value Decomposition}
\label{sec:lkp}

~\Cref{lower-bound-0} showed that a solution to $\kolp$ problem requires the error parameter $\varepsilon$ to be $O^*(1/\sqrt{d})$. \Cref{upper-bound-1} and~\Cref{upper-bound-2} give algorithms achieving this bound. However, for many polytopes with $k$ vertices, we can solve the $\kolp$ problem with $\varepsilon$ being $O^*(1/\sqrt{k})$. However, if $k<d$, this error is too high. To tackle this, we find a good approximation to the subspace spanned by the vertices of $K$, 
then we project to this subspace and use the result in~\Cref{upper-bound-1}. 
One such example is the ``Latent $k-$ Polytope'' (abbreviated $\lkp$) problem which we now describe.

The $\lkp$ problem has been studied in~\cite{BhattacharyyaK20}. Certain assumptions were made on the model, namely, the hidden polytope $K$ as well as on the (hidden) process for generating observed data from latent points in $K$. These assumptions are (a) shown to hold in several important Latent Variable models and (b) are sufficient to enable one to get polynomial time learning algorithms.

Here, we formulate assumptions which are similar, but, weaker in one important aspect. Whereas \cite{BhattacharyyaK20} assumed that each vertex of $K$ has a separation from the {\bf affine hull} of the other vertices (thus, in particular, each vertex is affinely independent of other vertices), we assume here that each vertex is separated only from the convex hull of the others. Under this weaker assumption, the algorithm of~\cite{BhattacharyyaK20} does not work. We give a different algorithm which we prove works. It is also simpler to state and carry out and its proof is based on a new general tool we introduce here - the Random Separating Hyperplane theorem (\Cref{thm:sephp}).  

\noindent
{\bf Assumptions on data in the $\lkp$ problem:} 
Let $\Mdot{1}, \ldots, \Mdot{k}$ denote the vertices of $K$ and $\bM$ be the $d \times k$ matrix with columns representing the vertices of $K$.  
We assume there are latent (hidden) points $P_{\cdot ,j}, j=1,2,\ldots ,n$ in $K$ and observed data points $A_{\cdot,j}, j=1,2,\ldots ,n$ are generated (not necessarily under any stochastic assumptions) by adding {\it displacements} $A_{\cdot,j}-P_{\cdot,j}$ respectively to $P_{\cdot,j}$. Clearly if the displacements are arbitrary, it is not possible to learn $K$ given only the observed data. So we need some bound on the displacements. 

Secondly, if all (or almost all) latent points lie in (or close to) the convex hull of a subset of $k-1$ or fewer vertices of $K$, the missing vertex cannot be learnt. To avoid this, we will assume that there is a certain $w_0$ fraction of latent points close to every vertex of $K$. 

Let \footnote{By the standard definition of spectral norm, it is easy to see that $\sigma_0^2$ is the maximum mean squared displacement in any direction.}
\begin{equation}
\label{eq:p}
\sigma_0 :=\frac{||\bP-\bA||}{\sqrt n}.
\end{equation}
We now show that the {\bf $\kolp$-Algorithm} mentioned in~\Cref{algorithmic-rsh}
has the desired properties: 
\begin{theorem}\label{LkP-Main-Thm}
Suppose $K$ is a latent poytope with $k$ vertices $M_{\cdot,1},M_{\cdot,2},\ldots ,M_{\cdot, k}$ and $\bP,\bA$ are latent points (all in $K$) and observed data respectively. Assume 
\begin{equation}\label{eq:cell}
\mbox{For all }\ell\in[k]\; ,\; C_\ell :=\{j:|P_{\cdot,j}-M_{\cdot,\ell}|\leq\frac{\sigma_0}{\sqrt{w_0}}\} \mbox{ satisfies    }
|C_\ell|\geq w_0n.
\end{equation}
Suppose $(\sqrt{\log k}/\sqrt{c_0k})\leq \delta\leq 1$ and $c_0$ is a large constant satisfying
\begin{equation}\label{1000}
\sigma_0\leq \frac{\delta^2\Delta(K)}{100c_0}\frac{\sqrt{w_0}}{\sqrt k}.
\end{equation}
Let $V$ be the $k$-dimensional SVD subspace of $\bf A$, and $\hatK$ denote the projection of $K$ on $V$. 
\begin{itemize}
    \item There is an $ \opto_{\frac{10\sigma_0}{\sqrt{w_0}\Delta}}(\hatK)$ oracle $\cal O$. 
    \item The algorithm {\bf $\kolp$ Algorithm} in~\Cref{algorithmic-rsh} outputs a set $Q$ of $k$ points such that the following condition is satisfied w.h.p.:  for every vertex $v$ of $K$, there is a point $v'$ in $Q$ with $|v-v'| \leq \delta \Delta(K)/5$. 
\end{itemize}
\end{theorem}

% {\color{red} The theorem needs to have a proof. Should be short based on existing claims and Lemmas. Pl put in.   ravi}
In the Theorem above, \eqref{1000}  implies an upper bound on $\sigma_0$ of $\Delta(K)\sqrt{w_0}/(c_0\sqrt k)$ and so the  oracle $\cal O$ used by the Theorem lies in $\opt_\varepsilon(\hatK)$ for  $\varepsilon\leq 1/(c_0\sqrt k)$. Thus, we get around the lower bound result~\Cref{lower-bound-0} by working in the $k$-dimensional SVD subspace of $\bA$. In rest of this section, we prove~\Cref{LkP-Main-Thm}. We begin by giving properties of the SVD subspace, and then show that an approximate optimization oracle exists in this subspace. Finally, we apply~\Cref{upper-bound-2kolp}.

\subsection{Properties of the SVD subspace of $\bA$}
%For sake of brevity assume that $\Delta(K)=1$. 
Let $V$ denote  the $k$-dimensional SVD subspace of $\bA$. We shall show that the projection of $K$ on $V$ has an   $\opto_\varepsilon(K)$ oracle for a suitable value of $\varepsilon$. Let $\hMdot{\ell}$ denote the projection of $\Mdot{\ell}$ on the SVD subspace $V$. Define $\hAdot{j}$ similarly, and let $\widehat{\bA}$ be the $d \times n$ matrix whose columns are given by $\hAdot{j}$. We now show that $\hMdot{\ell}$ and  $\Mdot{\ell}$ are close to each other. The following notation will turn out to be very useful in subsequent discussion: for a $d \times n$ matrix ${\bf B}$ and a subset $S$ of $[n]$, define 
$$ B_{\cdot, S} := \frac{1}{|S|} \sum_{j \in S} B_{\cdot, j}.$$

\begin{claim}
\label{cl:sum}
%Assume $\cE$ occurs. 
Let $S \subseteq [n]$ be a subset of indices. Then, $$|A_{\cdot, S} - P_{\cdot, S}| \leq \frac{\sigma_0 \sqrt n}{\sqrt{|S|}}. $$
\end{claim}
\begin{proof}
Define a unit vector $v \in \Re^d$ as follows: if $j \in S$, we define $v_j = \frac{1}{\sqrt{|S|}}$, 0 otherwise. Now, observe that 
$$|A_{\cdot, S} - P_{\cdot, S}| = \frac{1}{\sqrt{|S|}} \cdot |(\bA - \bP) \cdot v|. $$
Inequality~\eqref{eq:p} now implies that the RHS above is at most 
$\frac{\sigma_0 \sqrt{n}}{\sqrt{|S|}}.$
\end{proof}

\begin{lemma}
\label{lem:close}
%Assume that the event $\cE'$ happens. 
For all $\ell\in [k]$, 
$|\Mdot{\ell}-\hMdot{\ell}|\leq 5 \sigma_0/\sqrt{w_0} \leq \frac{\delta^2 \Delta(K)}{c_0}$.
\end{lemma}

\begin{proof}
We have $||\widehat{\bA}-\bA||\leq ||\bA-\bP||$ since, $\widehat{\bA}$ is the best rank $k$ approximation to $\bA$ in terms of the spectral norm and since each column of $\bP$ being a convex combination of the columns $\{ \Mdot{\cdot,\ell},\ell\in [k]\}$,
$\bP$ has rank at most $k$. We also have $||\widehat{\bA}-\widehat{\bP}||\leq ||\bA-\bP||$ since projections cannot increase length. Using these inequalities and the triangle inequality, we get: 
$$||\bP-\widehat{\bP}||\leq ||\bP-\bA||+||\bA-\widehat{\bA}||+||\widehat{\bA}-\widehat{\bP}||\leq 3 \cdot ||\bA-\bP|| \stackrel{\eqref{eq:p}}{\leq} 3\sigma_0\sqrt n.$$ 

Let $C_\ell$ be as in~\eqref{eq:cell}.  
 Now, define $w$ to be the unit vector with $w_j=1/\sqrt{|C_\ell|}, \forall j\in C_\ell$, and $w_j=0, \forall j\notin C_\ell$. We see that 
\begin{align*}
& |\Mdot{\ell}-\hMdot{\ell}|\leq \frac{|(\bP-\widehat{\bP})w|}{|C_\ell|}+\frac{2\sigma_0}{\sqrt{w_0}} 
\leq \frac{||\bP -\widehat {\bP }||}{\sqrt{|C_\ell |}} +\frac{2\sigma_0}{\sqrt{w_0}},
\end{align*}
 which proves the Lemma (using $|C_\ell| \geq w_0 n$). The second inequality in the Lemma follows from~\eqref{1000}.
\end{proof}

Let $\hatK$ be the projection of $K$ on $V$. We now show that $\hatK$ is also $\delta'$-well-separated for a value $\delta'$ close to $\delta$. 
\begin{lemma}
    \label{lem:hatK}
    The vertices of $\hatK$ are given by  $\{\hMdot{\ell}: \ell \in [k]\}$. Further, $ \hatK$ is $\delta'$-well-separated, where $\delta'= \delta \left( 1 - \frac{1}{100} \right)$. 
\end{lemma}
\begin{proof}
    It is clear that the vertex set of $\hatK$ is s subset of $S := \{\hMdot{\ell}: \ell \in [k]\}$. Thus, we need to show that none of the points in $S$ can be written as a convex combination of rest of the points in $S$. This will follow from the fact that $\hatK$ is $\delta'$-well-separated, and so it suffices to prove this statement. Consider a point $\hMdot{\ell} \in S$ and a point $p:= \sum_{\ell' \in [k], \ell' \neq \ell} \lambda_{\ell'} \hMdot{\ell'},$ which is in $\CH(S \setminus\{\hMdot{\ell}\}).$ Then 
    \begin{align*}
        |p-\hMdot{\ell}| & = \left| \sum_{\ell' \neq \ell} \lambda_{\ell'} (\hMdot{\ell'} - \hMdot{\ell} ) \right| \\
        & \geq \left| \sum_{\ell' \neq \ell} \lambda_{\ell'} (\Mdot{\ell'} - \Mdot{\ell} ) \right| - \sum_{\ell' \neq \ell} \lambda_{\ell'} |\Mdot{\ell'}-\hMdot{\ell'}| - 
       \sum_{\ell' \neq \ell} \lambda_{\ell'} |\Mdot{\ell} - \hMdot{\ell}| \\
        & = \left| \sum_{\ell' \neq \ell} \lambda_{\ell'} \Mdot{\ell'} - \Mdot{\ell} \right| - \frac{2\delta^2 \Delta(K)}{c_0} \\
        & \geq (\delta - \frac{2\delta^2}{c_0}) \Delta(K) \geq \delta' \Delta(K)
    \end{align*}
    where the second last line uses~\Cref{lem:close} and last line follows from  the fact that $K$ is $\delta$-well-separated. Note that $\Delta(\hatK) \leq \Delta(K).$ Thus it follows that $\hatK$ is is $\delta'$-well-separated.
\end{proof}

% {\color{red} Below in several places, $\varepsilon$ has dimensions of length. It really should be dimensionless. Think, we just have to divide the values set for $\varepsilon$ by $\Delta(K)$ for all occurences of $\varepsilon$ till the end of the section.}

We would now like to use~\Cref{upper-bound-2} on $\hatK$ with $\varepsilon = \frac{\sigma_0}{\sqrt{w_0}\Delta(K)}$ and $\delta'=\delta$. Indeed the condition~\eqref{1000} along with~\Cref{lem:hatK} show that the conditions of~\Cref{upper-bound-2} hold provided we are able to exhibit an efficient $\opto_\varepsilon(\hatK)$ oracle. 

\subsection{Construction of the $\opto_\varepsilon(\hatK)$ oracle}

We now describe the construction for the   $\opto_\varepsilon(\hatK)$ oracle in the subspace $V$, where $\varepsilon = \frac{10 \sigma_0}{\Delta(K) \sqrt{w_0}}$. Let $u$ be a unit vector in $V$. The procedure (referred as {\bf Subset Smoothing Algorithm} in~\Cref{algorithmic-rsh}) is given in~\Cref{algo:center}  -- we project the points in $\hatA$ on $u$ and consider the  $w_0 n$ with the highest projection along $u$. Finally, we output the average of these points.

\begin{algorithm}[H]
  \caption{Orcale in $\opto_\varepsilon(\hatK)$}
  \label{algo:center}
    {\bf Input:} unit vector $u \in V$. \;
       Let $R(u)$ be the index set of the columns $\hAdot{j}$ with the $w_0 n$ highest values of $u \cdot \hAdot{j}$. \;
      {\bf Output:} $\hAdot{R(u)}$ \;
\end{algorithm}

We now show that this algorithm has the desired properties. Let $R(u)$ be the index set as in~\Cref{algo:center}. The following is an easy consequence of~\Cref{cl:sum}. 
\begin{claim}
\label{cl:eps1}
    $\dst(\hAdot{R(u)}, \hatK) \leq \varepsilon\Delta(K) . $
\end{claim}
\begin{proof}
    \Cref{cl:sum} shows that $|\Adot{R(u)} - \Pdot{R(u)}| \leq \varepsilon\Delta(K) .$ Therefore, 
    $|\hAdot{R(u)} - \hPdot{R(u)}| \leq \varepsilon\Delta(K) $. The desired result now follows because
    $\hPdot{R(u)} \in \hatK.$
\end{proof}

Let $\ell \in [k]$ be the index for which $\hMdot{\ell} \cdot u$ is maximized. 
\begin{claim}
    \label{cl:eps2}
    $\hAdot{R(u)} \cdot u \geq \hMdot{\ell} \cdot u - \varepsilon\Delta(K) .$
\end{claim}
\begin{proof}
    Let $C_\ell$ be the index set specified by~\eqref{eq:cell}. It suffices to show that 
    $$ \hAdot{C_\ell} \cdot u \geq \hMdot{\ell} \cdot u - \varepsilon\Delta(K) .$$ 
    Now, 
    \begin{align*}
    \hMdot{\ell} \cdot u - \hAdot{C_\ell} \cdot u & \leq |\hMdot{\ell} - \hAdot{C_\ell}| \\
    & \leq |\hMdot{\ell} - \hPdot{C_\ell}| + |\hPdot{C_\ell} - \hAdot{C_\ell}| \\
    & \leq |\Mdot{\ell} - \Pdot{C_\ell}| + |\Pdot{C_\ell} - \Adot{C_\ell}| \\
    & \leq \varepsilon \Delta(K),
    \end{align*}
where the last inequality follows from~\eqref{eq:cell} and~\Cref{cl:sum}. 
\end{proof}

The above two results show that~\Cref{algo:center} yields an  $\opto_\varepsilon(\hatK)$ oracle.  We would now like to apply~\Cref{upper-bound-2kolp} to the polytope $\hatK$ with parameters $\varepsilon$ and $\delta' = \delta(1-1/100).$~\Cref{lem:hatK} shows that $\hatK$ is $\delta'$-well-separated. We now need to check that the parameters $\varepsilon$ and $\delta'$ satisfy the following conditions needed in the statement of~\Cref{upper-bound-2kolp} (recall that $\hatK$ is a polytope in $\Re^k$): (i)  $\delta^2 \geq c \varepsilon \sqrt{k}$, and (ii) $\delta \geq \frac{\sqrt{\log k}}{\sqrt{c} k}.$
The second condition is already an assumption in the statement of~\Cref{LkP-Main-Thm}, and the first condition follows from~\eqref{1000} and the fact that $\varepsilon = \frac{10 \sigma_0}{\Delta(K) \sqrt{w_0}}$. Applying~\Cref{upper-bound-2kolp}, we get a set of $k$ points $Q$, such that for each vertex $v'' $ of $\hatK$, there is a point $v' \in Q$ with $|v''-v'| \leq \delta \Delta(K')/10$. Applying~\Cref{lem:close} proves~\Cref{LkP-Main-Thm}.

\section{Open Problems}
\label{sec:open}
We now mention some problems that remain open in our work:
\begin{itemize}
    \item[(i)] In the statement of $\rsh$,
    the success probability of the desired event is $O\left(1/k^{O(1/\delta^2)}\right)$. Can we 
    improve the exponential dependence of the success probability on $1/\delta$?
\item[(ii)] Prove an analog (under suitable assumptions) of~\Cref{thm-sep-to-opt} for other reductions among oracles. Of particular interest is a reduction from an Optimization oracle to a Separation Oracle.
\item[(iii)]\Cref{upper-bound-3} on the $\Haus$ problem returns exponentially many points whose convex hull approximates $K$. Can this be improved, either via a an improvement mentioned in the first open problem above, or alternatively, feeding the exponentially many points to the algorithm of~\cite{BlumHR19}? 
\end{itemize}

{\small
\bibliographystyle{alpha}
\bibliography{ref}
}

\appendix
\section*{Appendix}
\subsection*{{\bf A.} Some Examples}
\label{sec:examples}
The first example shows that the factor $1/\sqrt{m}$ in the statement of $\rsh$ (\Cref{rsh-0}) cannot be improved, even for a simple polytope $K$ consisting of just one line segment. 

\vspace*{0.1 in}
\noindent
{\bf Example 1:} Suppose  $K$  is a polytope in $\Re^d$ with just two vertices $\Delta$ distance apart and $V = \Re^d, \delta=1$. It is easy to verify from standard results on random projection on a line that the additive term in inequality~\eqref{eq:rsh} is tight up to constant factors.

\vspace*{0.1 in}
The next example shows that the success probability of the desired event~\eqref{eq:rsh} in~\Cref{rsh-0} needs to depend on $k$, the number of vertices of the polytope $K$. 

\vspace*{0.1 in}
\noindent
{\bf Example 2:} Consider the Euclidean space $\Re^d$ and the $d-1$ dimensional sphere $K:=\{x \in \Re^d: |x|\leq 1, x_1=0\}$. Observe that the diameter $\dia$ of $K$ is 1. 
Let $a$ denote the point $(1,0,0,\ldots ,0)$. It is at distance 1 from $K$, and so the parameter $\delta$ in the statement of~\Cref{rsh-0} can be set to 1.  But, with high probability,  a random unit length vector $u$ has $u\cdot a\approx \frac{c}{\sqrt d}$ for some constant $c$,  whereas the maximum of  $v\cdot x$ over all points $x$ in the sphere $K$ is about 1. Therefore the probability of the event~\eqref{eq:rsh} is exponentially small in $d$.
\end{document}